\documentclass[preprint,authoryear,12p,times]{elsarticle}
\usepackage[top=2cm, bottom=2cm, left=2cm, right=2cm]{geometry}
\usepackage{makecell}
\usepackage{booktabs}
\usepackage{amssymb,amsthm}
\usepackage{graphicx}
\usepackage{amsmath, bm}
\usepackage{mathrsfs}
\usepackage{algorithm, algorithmic}
\usepackage{multirow}
\usepackage{multicol}
\usepackage{balance}
\usepackage{setspace}
\usepackage{titlesec} 
\usepackage{comment}
\usepackage{subcaption}
\usepackage{caption}
\usepackage{marvosym} 
\usepackage{lineno}

\usepackage{xcolor}
\usepackage{pifont}
\usepackage[colorlinks=true,linkcolor=cyan,citecolor=orange]{hyperref}

\newtheorem{thm}{Theorem}
\newtheorem{lemma}{Lemma}

 \def\BibTeX{{\rm B\kern-.05em{\sc i\kern-.025em b}\kern-.08em
     \top\kern-.1667em\lower.7ex\hbox{E}\kern-.125emX}}

\def\MyMethod{U-Face} 

\def\FullMyMethod{Unsupervised Facial Attribute Controllable Editing} 
\def\FullMyAlgorithm{Alternating Iterative Disentanglement and Controllability} 
\def\MyAlgorithm{AIDC} 

\def\InterFaceGAN{InterFaceGAN}
\def\GANspace{GANSpace}
\def\sefa{SeFa}
\def\enjoyGAN{EnjoyGAN}
\def\MDSE{MDSE}
\def\AdaTans{AdaTrans} 
\def\SDFlow{SDFlow}

\def\B#1{\mathbf#1}
\def\C#1{\mathbb#1}
\begin{document} 

\begin{frontmatter}
  \title{U-Face: An Efficient and Generalizable Framework for Unsupervised Facial Attribute Editing via Subspace Learning}

  \author{Bo Liu$^{a,b}$, Xuan Cui$^{a,b}$, Run Zeng$^{a,b}$, Wei Duan$^{a,b}$, Chongwen Liu$^{a,b}$, Jinrui Qian$^{a,b}$, Lianggui Tang$^{a,b}$, Hongping Gan$^{c,*}$}
  \address[1]{Chongqing Key Laboratory of Intelligent Sensing and Blockchain Technology, Chongqing Technology and Business University, Chongqing, China}
  \address[2]{School of Artificial Intelligence, Chongqing Technology and Business University, Chongqing, China}
  \address[3]{School of Software, Northwestern Polytechnical University, Xi'an, China}

\begin{abstract}
Latent space-based facial attribute editing methods have gained popularity in applications such as digital entertainment, virtual avatar creation, and human-computer interaction systems due to their potential for efficient and flexible attribute manipulation, particularly for continuous edits. Among these, unsupervised latent space-based methods, which discover effective semantic vectors without relying on labeled data, have attracted considerable attention in the research community. However, existing methods still encounter difficulties in disentanglement, as manipulating a specific facial attribute may unintentionally affect other attributes, complicating fine-grained controllability. To address these challenges, we propose a novel framework designed to offer an effective and adaptable solution for unsupervised facial attribute editing, called \FullMyMethod\ (\MyMethod). The proposed method frames semantic vector learning as a subspace learning problem, where latent vectors are approximated within a lower-dimensional semantic subspace spanned by a semantic vector matrix. This formulation can also be equivalently interpreted from a projection-reconstruction perspective and further generalized into an autoencoder framework, providing a foundation that can support disentangled representation learning in a flexible manner. To improve disentanglement and controllability, we impose orthogonal {non-negative} constraints on the semantic vectors and incorporate attribute boundary vectors to reduce entanglement in the learned directions. Although these constraints make the optimization problem challenging, we design an alternating iterative algorithm, called \FullMyAlgorithm\ (\MyAlgorithm), with closed-form updates and provable convergence under specific conditions. Extensive experiments on multiple pre-trained GAN generators demonstrate that \MyMethod\ {shows improved} disentanglement, controllability, and visual fidelity compared to existing state-of-the-art supervised and unsupervised baselines. Specifically, \MyMethod\ reduces inter-attribute correlation {by an average of} $5$\%{-}$15$\%, and {improves} FID {by $10$\%{-}$20$\%} and LPIPS {by $5$\%{-}$10$\%}. {Compared to recent diffusion-based methods (DDS, CDS, FPE), \MyMethod\ achieves comparable editing quality and reduces inference times}, making it suitable for {real-time} applications across different GAN backbones and enabling large-scale interactive facial editing.

\end{abstract}

\begin{keyword}
  Facial attribute editing, Attribute boundary vector, Generative adversarial networks, Orthogonal {non-negative} constraint
\end{keyword}
\end{frontmatter}

\section{Introduction}
Facial attribute editing~\citep{GUO2024108683, REN2025127245} involves manipulating specific attributes of a facial image, such as age, gender, hairstyle, and expression, while preserving the overall identity and natural appearance of the face~\citep{TPAMI2024FGE}. This technique has found significant application in areas such as digital entertainment, virtual avatar generation, and human-computer interaction systems. Early studies on makeup transfer~\citep{SUN2024109346}, expression editing~\citep{Li.2012,TIAN20181}, gender transformation~\citep{Suo.2011}, and age progression~\citep{KemelmacherShlizerman.2014} relied on classical computer vision~\citep{PatchNet2013} and machine learning~\citep{Garrido.2014} techniques, and while significant progress has been made, challenges in achieving effective disentanglement and fine-grained control remain.\par
{In recent years, Generative Adversarial Networks (GANs) have made certain advancements in facial attribute editing.} Existing GAN-based methods are typically categorized into two main types: image space-based methods, which work directly in image space, and latent space-based methods, which operate in a learned latent space. Both have shown promising results, but each comes with its own set of challenges. Image space-based methods~\citep{AttGAN2019, StarGAN2019, StyleTrans2021, UVCGAN2022, TPAMI2023VecGAN, Huang2024SDGANDS, SHAO2021107311} leverage encoder-decoder architectures and have achieved promising results, but high computational costs and challenges in supporting continuous attribute editing remain significant limitations for many image space-based methods. In contrast, latent space-based methods exploit semantic vectors in the latent space of pre-trained GAN models (e.g., StyleGAN~\citep{StyleGAN2019}, ProGAN~\citep{ProGAN2018}, DFSGAN~\citep{YANG2023105519} and His-GAN~\citep{LI201931}), offering potentially more efficient and flexible attribute manipulation, particularly for continuous edits.\par
Latent space-based methods can be supervised or unsupervised depending on the use of labeled data. Supervised latent space-based methods~\citep{Van2021,Davis.2022,MDSE2023,InterFaceGAN2020,EnjoyGAN2021} rely on labeled data or pre-trained models to guide the training of semantic vectors. For example,~\cite{MaskFaceGAN2023} derived semantic vectors using an attribute classifier and a face parser, enabling focus on localized attribute editing. ~\cite{STGAN2019} used the CelebA dataset with attribute labels to learn semantic vectors for facial attribute editing. While these methods facilitate the efficient generation of high-quality semantic vectors, the reliance on labeled data can increase training costs and pose challenges in achieving effective disentanglement and fine-grained controllability.\par
In contrast, unsupervised latent space-based methods~\citep{LI2023103916,AdaTrans2023,AGE2022CVPR,TIP2022STIAWO,SEFA2021, DOGAN2020338} learn semantic vectors without relying on labeled data. {Although these methods demonstrate promising performance, the lack of explicit supervision often leads to entangled edits and limited controllability, which can pose challenges in certain real-world applications—such as virtual avatars and interactive entertainment—where fine control and flexibility are crucial. }{To address these issues, we propose U-Face, an unsupervised facial attribute controllable editing method. With its efficient unsupervised learning mechanism, U-Face operates stably in real-time systems, making it suitable not only for virtual avatars and interactive entertainment applications, but also highly promising for efficient AI processing in resource-constrained or IoT scenarios, such as 3D object detection~\citep{11048915}.} Our main contributions are as follows:
\begin{itemize}
  \item  We propose a novel framework, \FullMyMethod\ (\MyMethod), for unsupervised facial attribute editing. It substantially improves disentanglement and controllability while preserving high visual fidelity. In addition, \MyMethod\ provides a principled foundation for unsupervised disentangled representation learning, with the potential to be extended to more expressive neural network architectures such as autoencoders.

  \item  We formulate semantic vector learning as a subspace learning problem, where latent vectors are approximated within a semantic subspace. To strengthen disentanglement and controllability, we introduce orthogonal {non-negative} constraints that enforce independence and sparsity among semantic vectors, and incorporate attribute boundary vectors to further reduce entanglement in facial attribute manipulation.

  \item We design an efficient alternating optimization algorithm, \FullMyAlgorithm\ (\MyAlgorithm), which admits closed-form solutions for each subproblem, ensuring fast convergence with provable guarantees.
  \item We conduct extensive experiments across multiple pre-trained GANs, benchmarking against state-of-the-art unsupervised, supervised, and diffusion-based baselines using five evaluation metrics, including FID and LPIPS. Furthermore, we provide a comprehensive analysis covering identity consistency, parameter sensitivity, user studies, and visualizations of challenging cases.
\end{itemize}

The structure of this paper is organized as follows: Section~\ref{sec:RelWorks} reviews previous work on latent space-based facial image editing. Section~\ref{sec:method} presents the proposed \MyMethod\ framework. Section~\ref{sec:exp} and~\ref{sec51} provide experimental results and discussions, respectively. Section~\ref{sec:conclusion} concludes the study and outlines potential directions for future research. 
\section{Related Work}\label{sec:RelWorks}
\subsection{Supervised Latent Space-Based Facial Image Editing}
Supervised latent space-based methods leverage pre-trained models and labeled data to learn directions in the latent space that enable facial attribute manipulation. Recently, several methods have used labeled data to guide facial attribute editing. For example,~\cite{Van2021} incorporated human prior knowledge into a model to provide more information regarding the GAN latent space. To demonstrate the feasibility of using brain signals for image editing.~\cite{Davis.2022} recorded electroencephalogram responses from participants viewing specific semantic features, and used these responses as supervision signals to learn semantic features in the GAN latent space.\par
Additionally, obtaining labeled data from the latent space and calculating semantic vectors for editing facial attributes has been extensively studied. For example,~\cite{MDSE2023} utilized image data with attribute labels to construct an orthogonal subspace, where each subspace contained semantic vectors that edited specific facial attributes. These orthogonal semantic vectors effectively mitigated the entanglement problem during facial attribute editing. Similarly,~\cite{EnjoyGAN2021} trained a linear attribute regressor to predict facial attribute scores, improving disentanglement effects.~\cite{InterFaceGAN2020} trained a Support Vector Machine (SVM) classification hyperplane using image data with attribute labels, thereby enabling the separation of two completely opposing facial attributes for precise controlled editing. Furthermore, supervised deep learning techniques, such as pre-trained ResNet and Multilayer Perceptron (MLP), have been employed to train semantic vector models~\citep{IALS2021IJCAI, STGAN2019, Yao2021ICCV, MaskFaceGAN2023}.\par
Although these methods exhibit excellent performance in learning semantic vectors, they rely heavily on labeled data and pre-trained models. This dependency increases the complexity of training, but also introduces significant time overhead. Moreover, semantic vectors obtained using these methods suffer from limited controllability and entanglement problems.

\subsection{Unsupervised Latent Space-Based Facial Image Editing}
Unsupervised latent space-based methods for facial attribute editing eliminate the need for labeled data by discovering directions directly from the latent space. Some of these methods exploit the underlying structure of the latent space to manipulate facial attributes. For example,~\cite{ICLR2020Steer} employed data augmentation to explore the controllability of the latent space.~\cite{VoynovB20, Cherepkov.2021, ICLR2021Regres} identified semantically meaningful vectors and compositional attributes within GAN's latent space, enabling manipulation of synthetic image attributes.~\cite{ICLR2022Escape} proposed an unsupervised method called Local Basis, which leveraged the local geometry of the intermediate latent space in GANs to identify semantic decomposition directions.~\cite{TIP2022STIAWO} introduced a structure-texture independent architecture with weight decomposition and orthogonal regularization techniques to disentangle the latent space.~\cite{AGE2022CVPR} proposed an attribute group editing method using sparse dictionary learning to manipulate category-independent attributes while preserving category-dependent attributes globally.\par
Other studies have applied various unsupervised techniques, such as PCA, to discover semantic vectors in latent space. For example,~\cite{GANSpace2020} proposed an unsupervised facial image editing method that utilized PCA to compute semantic vectors in the latent space, allowing the manipulation of facial attributes such as expressions and age.~\cite{SEFA2021} achieved precise control and disentanglement of facial attribute editing by applying eigen decomposition to the weight matrix of a pre-trained GAN model.~\cite{AdaTrans2023} presented an adaptive nonlinear latent transformation to identify semantic vectors in the latent space using an unsupervised strategy. This method maximized the log-likelihood of the transformed latent vectors to facilitate the distribution transformations.~\cite{SDFlow2024ICASSP} utilized unsupervised conditional continuous normalizing flows to identify semantic variables for fine-grained attribute editing in the latent space of StyleGAN.\par 
However, without large-scale labeled data, existing unsupervised latent space-based facial attribute editing methods still suffer from severe attribute entanglement and limited controllability. To address these issues, we propose a novel unsupervised framework that integrates the orthogonal constraint with the facial attribute boundary vector. This framework enables the identification of semantic vectors that are both well disentangled and highly controllable, as supported by theoretical analysis and experimental results.\par

\subsection{Diffusion-based Facial Image Editing}
Diffusion models have recently gained considerable attention in facial attribute editing due to their strong generative capacity and controllability.~\cite{DiffusionCLIP2022} integrates CLIP guidance into diffusion models, enabling semantic attribute manipulation through text prompts without the need for retraining. \cite{Chen2023FaceAV} fine-tuned diffusion models with age-aware objectives to achieve accurate and realistic face aging while preserving identity. Building on this idea,~\cite{dds2023} proposed to directly adjust the denoising score function, achieving more precise attribute-level control in an efficient manner. More recently,~\cite{FPE2024} provided an in-depth analysis of how cross-attention and self-attention layers in diffusion models capture semantic correspondences, offering new insights and strategies for localized and disentangled editing. Similarly,~\cite{cds2024} introduced a contrastive denoising objective, improving the alignment between text guidance and visual edits while enhancing the fidelity of generated results.~\cite{Wang2024DiffFAEAH} introduced a one-shot diffusion-based framework with 3DMM-guided customization and semantic composition for high-fidelity facial appearance editing. {\cite{RetrievFace2025} proposed a text-guided method that combines retrieval-enhanced and diffusion models to achieve precise control and editing of facial attributes by retrieving relevant image features. \cite{Wei2025MagicFaceHF} proposed a diffusion model-based method for high-fidelity facial expression editing, enabling fine-grained, continuous, and interpretable expression manipulation while maintaining identity, pose, background, and facial detail consistency. \cite{hou2025zeroshotfaceeditingidattribute} proposed a zero-shot face editing method that uses a diffusion model to decouple identity and attribute features, enabling independent control over facial attributes while preserving identity consistency. \cite{DiT4Edit2025} proposed a Diffusion Transformer-based image editing framework that utilizes the DPM-Solver inversion algorithm and unified attention control to enhance the quality and speed of image editing.}\par

While they highlight the potential of diffusion models for high-quality and semantically aligned facial attribute editing, they generally involve heavy computation and slow inference due to iterative sampling. Moreover, their disentanglement ability remains limited when fine-grained, attribute-specific control is required without text supervision. In contrast, our proposed \MyMethod\ provides efficient and disentangled attribute editing with competitive visual fidelity, complementing the strengths of diffusion-based methods.

\section{Proposed Framework}\label{sec:method} 
This section first introduces a simple, effective and generalizable subspace learning framework for unsupervised semantic vector discovery in latent space facial attribute editing. To address the limitations in disentanglement and controllability, we propose a novel framework,~\MyMethod, as illustrated in Fig.~\ref{framework}. The objective function of~\MyMethod\ incorporates orthogonal {non-negative} constraints and attribute boundary vectors to enhance attribute disentanglement and controllability. To solve the proposed objective, we further develop an alternating iterative algorithm,~\FullMyAlgorithm\ (\MyAlgorithm). We further provide a theoretical guarantee of convergence for \MyAlgorithm\ and analyze its computational complexity. {Compared to related subspace-based methods: PCA/GANSpace extracts variance-maximizing axes without semantic anchoring and lacks a mechanism to reduce attribute coupling; SeFa derives global directions via eigen-decomposition of generator weights, but similarly lacks disentanglement-oriented joint constraints; InterFaceGAN learns independent SVM boundaries for each attribute and does not jointly control inter-attribute interference; although MDSE also incorporates orthogonality constraints, it decomposes the latent space into multiple orthogonal subspaces and optimizes at the subspace level, whereas U-Face optimizes only at the vector level, avoiding excessive decomposition of the subspace. In contrast to these existing methods, U-Face jointly optimizes a set of semantic vectors under explicit constraints, which helps reduce inter-attribute entanglement.}

\begin{figure*}[t]
  \centering
  \includegraphics[width=\linewidth]{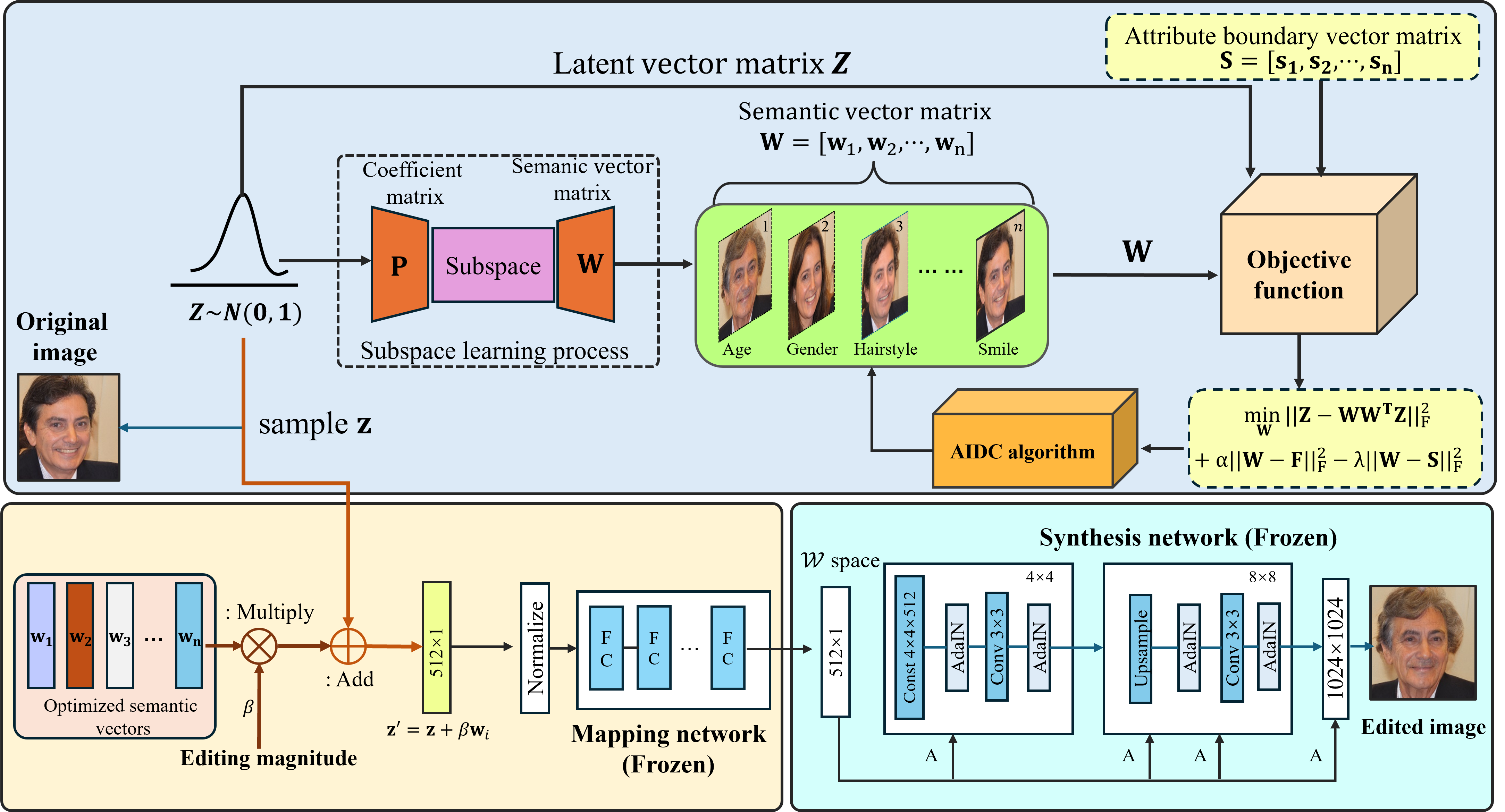}
  \vspace{-10pt}
  \caption{\FullMyMethod\ (\MyMethod) framework.}
~\label{framework}
\end{figure*}

\subsection{Latent Vector Subspace Learning}
~\label{sec31}
A \emph{semantic vector} in the latent space of a GAN refers to a direction along which facial attributes (e.g., adding a smile or glasses) can be manipulated. Specifically, a latent vector $\B{z} \in \mathbb{R}^m$ is sampled from the $m$-dimensional latent space $\mathcal{X}$ of a pre-trained GAN such as StyleGAN2~\citep{Karras.2020}. The generator $G$ then maps $\B{z}$ to the corresponding facial image. To manipulate a specific attribute, a semantic vector $\B{w}$ is added to the latent vector, and the manipulated image is obtained by $G(\B{z} + \B{w})$. Hence, identifying meaningful semantic vectors $\B{w}$ is crucial for effective facial attribute editing.\par

The subspace learning paradigm~\citep{SP-DiffFusion2025, OSR2025} is a widely used unsupervised learning method. 
Inspired by this idea, we propose a lightweight method for learning semantic vectors without the need for labeled data or complex neural architectures. 
The key idea is to approximate latent vectors within a lower-dimensional semantic subspace, thereby encouraging the learned semantic vector matrix to capture disentangled and controllable semantic directions.\par

Let $\B{Z} = [\B{z}_1, \B{z}_2, \ldots, \B{z}_n] \in \mathbb{R}^{m \times n}$ be a latent vector matrix composed of $n$ latent vectors sampled from the latent space of a GAN model. We aim to learn a semantic vector matrix $\B{W} = [\B{w}_1, \B{w}_2, \ldots, \B{w}_r] \in \mathbb{R}^{m \times k}$ together with a coefficient matrix $\B{P}\in\mathbb{R}^{k\times n} (k<m) $ such that $\B{Z}\approx \B{W}\B{P}$. This leads to the following subspace learning objective:
  \begin{equation}
    \label{classicPCA}
    \begin{split}
      & \min_{\B{W},\B{P}} \|\B{Z}-\B{W}\B{P}\|_\text{F}^2,\\
      &  s.t \quad \|\B{W}\|_\text{F} = 1,
    \end{split}
    \end{equation}
where the Frobenius norm $\|\cdot \|_\text{F}$ constraint prevents trivial solutions (e.g., $\B{W}=0$), with $\B{W}$ denoting the semantic vector matrix, whose columns span a low-dimensional semantic subspace and $\B{P}$ representing the low-dimensional codes of latent vectors, each column of $\B{P}$ specifies how a latent vector is expressed as a linear combination of semantic vectors. This view aligns with matrix factorization and dictionary learning, where latent vectors are represented in a low-dimensional semantic subspace spanned by $\B{W}$ with coefficients $\B{P}$.\par

Eq.~\eqref{classicPCA} can also be interpreted from the projection-reconstruction perspective: 
$\B{P}$ serves as the projected codes of latent vectors into a semantic subspace, while $\B{W}$ reconstructs the original latent vectors from these codes. 
Equivalently, the formulation can be regarded as a shallow linear autoencoder: 
$\B{P}$ functions as the encoder that maps latent vectors into low-dimensional representations, and $\B{W}$ acts as the decoder that reconstructs them. Such an interpretation naturally makes the framework extensible by replacing the linear encoder/decoder with nonlinear neural networks, it can be generalized into a deep autoencoder for disentangled representation learning~\citep{DRL2024}.\par

Since Eq.~\eqref{classicPCA} is a non-convex optimization problem, $\B{W}$ is typically solved via alternating optimization, yielding a local optimum. However, the semantic vectors obtained from this objective often exhibit attribute entanglement, i.e., a single semantic vector may simultaneously influence multiple attributes. As illustrated in Fig.~\ref{PCAVector}, applying $\B{w}_1, \ldots, \B{w}_5$ results in edited images where each semantic vector influences more than one attribute. For example, $\B{w}_1$ modifies glasses and age, while $\B{w}_2$ manipulates glasses and hairstyle. In addition, this entanglement substantially limits controllability in attribute manipulation.\par
To overcome these challenges, we propose an unsupervised facial attribute editing framework based on orthogonal {non-negative} constraints and attribute boundary vectors.\par

\textbf{Motivation and Intuition.}
{For clarity, we first present a concise conceptual overview of the design choice. Our goal is to learn $k$ semantic directions that behave like independent sliders in latent space: moving one slider (e.g., "Smile") should not unintentionally change another (e.g., "Age"). Three ingredients enable this:}
\begin{itemize}
  \item {\emph{Orthogonality ($\mathbf{W}^\top\mathbf{W}=\mathbf{I}$)} enforces independence across sliders. Geometrically, the columns of $\mathbf{W}$ are perpendicular; nudging one direction minimally projects onto the others, reducing cross-attribute interference.}
  \item {\emph{Non-negativity ($\mathbf{W}\!\ge\!\mathbf{0}$)} encourages parts-based, additive use of latent coordinates, which empirically yields sparser, more interpretable directions and reduces intra-vector mixing.}
  \item {\emph{Boundary vectors ($\mathbf{S}$)} serve as "do-not-lean-toward" references. Each $\mathbf{s}_i$ encodes the normal of a decision boundary for one attribute. Our regularizer $-\lambda\|\mathbf{W}-\mathbf{S}\|_{F}^{2}$ repels the learned directions from these boundaries, for example, helping a "Hairstyle" slider avoid entanglement with the "Gender” boundary.}
\end{itemize}
\begin{figure*}[t]
  \centering
  \includegraphics[width=\linewidth]{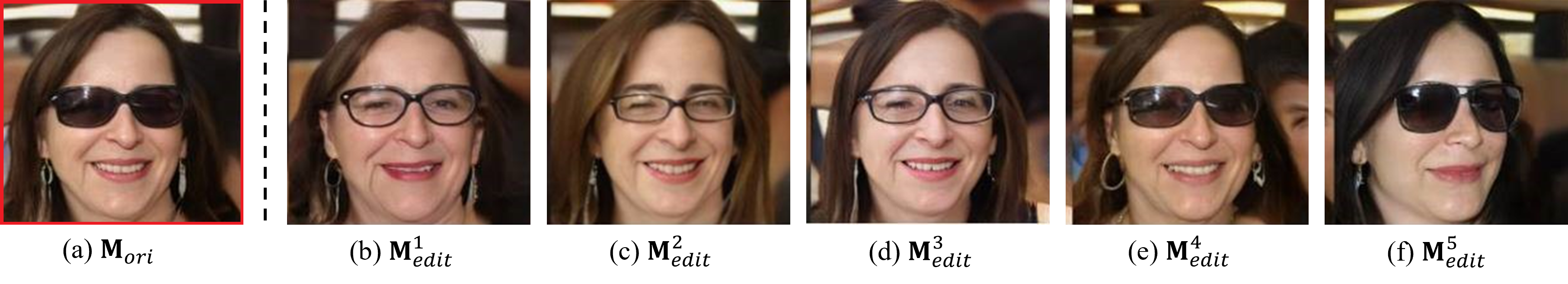}
  \vspace{-10pt}
  \caption{$\B{w}_1, \B{w}_2, \ldots, \B{w}_5$ are the five semantic vectors.~(a) shows the original facial image $\B{M}_{ori} = G(\B{z})$, and (b)-(f) show edited images $\B{M}_{edit}^i = G(\B{z} + \B{w}_i),i=1,2,\ldots,5$. Each semantic vector simultaneously manipulates multiple facial attributes: $\B{w}_1$ adjusts both glasses and age, $\B{w}_2$ manipulates glasses and hairstyle, and the others similarly influence various attributes.}
~\label{PCAVector}
\end{figure*}

\subsection{The Proposed Objective Function of \MyMethod}
To improve the disentanglement and controllability of semantic vectors obtained from Eq.~\eqref{classicPCA}, \MyMethod\ introduces orthogonal {non-negative} constraints into the objective. Furthermore, to enhance the disentanglement of attribute manipulation, we incorporate attribute boundary vectors into the formulation.

\subsubsection{Orthogonal {non-negative} Constraints}
Although Eq.~\eqref{classicPCA} provides a basic formulation for computing semantic vectors, its optimization is often prone to local minima and tends to produce entangled or overlapping directions. To mitigate this issue, we enforce orthogonality constraints among the column vectors of the semantic vector matrix $\B{W}$, ensuring that each semantic vector lies in an independent subspace. These constraints ensure that each semantic vector lies in an orthogonal subspace. Moreover, we impose a non-negativity condition, which encourages sparsity within each semantic vector and further reduces intra-vector entanglement, making the learned semantic vectors more interpretable and controllable:
\begin{equation}
\label{objFun111}
\begin{split}
  & \min_{\B{W},\B{P}} \|\B{Z}-\B{W}\B{P}\|_\text{F}^2,
  \\
  & s.t \quad \B{W}^\top\B{W}=\B{I},
\end{split}
\end{equation}
where $\B{I} \in \mathbb{R}^{k \times k}$ denotes the identity matrix. Since the orthogonality constraint $\B{W}^\top\B{W}=\B{I}$ in Eq.~\eqref{objFun111} prevents trivial solutions, we do not incorporate the constraint $\|\B{W}\|_\text{F} = 1$ used in Eq.~\eqref{classicPCA} into Eq.~\eqref{objFun111}.\par

While the orthogonality constraint improves disentanglement across semantic vectors, we further introduce sparsity to enhance intra-vector disentanglement by reducing interactions among the elements within each semantic vector. As stated in Theorem~1 of~\cite{KBS2024ANSC}, the matrix $\B{W}$ that satisfies the orthogonal {non-negative} condition tends to be sparse in its columns. Therefore, we only introduce a non-negativity constraint, $\B{W} \geq 0 $, into Eq.~\eqref{objFun111} and reformulate the objective of \MyMethod\ as follows:
  \begin{equation}
  \label{objFun2}
  \begin{split}
    & \min_{\B{W},\B{P}} \|\B{Z}-\B{W}\B{P}\|_\text{F}^2,
    \\
    & s.t \quad \B{W}^\top\B{W}=\B{I},\B{W}\geq \B{0}.
  \end{split}
  \end{equation}

\subsubsection{Attribute Boundary Vector}
To further enhance disentanglement and establish a direct correspondence between semantic vectors and specific facial attributes, we incorporate a reference direction termed the \textbf{attribute boundary vector}. Intuitively, each boundary vector encodes the decision hyperplane separating the presence and absence of an attribute in latent space. By encouraging semantic vectors to diverge from these boundaries, the learned directions become less entangled with correlated attributes and more aligned with the intended editing effect.

Each attribute boundary vector is associated with the corresponding semantic vector, possessing the following properties:
\begin{itemize}
  \item  The effect of facial attribute editing becomes more pronounced as the semantic vector moves in a direction perpendicular to the attribute boundary vector.
  \item If the semantic vector moves in the direction opposite of the attribute boundary vector, the corresponding attribute in the image is reversed. For example, an image with a smiling face may be transformed into an image with a crying face.
\end{itemize}

Therefore, the disentanglement and controllability of attributes can be enhanced by moving the semantic vector away from the corresponding attribute boundary vector. Simultaneously, the degree of manipulation can be precisely controlled by adjusting through the editing magnitude. For each facial attribute, the corresponding attribute boundary vector can be computed using the method proposed in~\cite{semanticBound2020}.\par
For the semantic vector matrix $\B{W} = {[\B{w}_1, \B{w}_2, \ldots, \B{w}_k]} \in \C{R}^{m \times k}$ and the corresponding $k$ attribute boundary vectors, with $m$ dimensions, we construct the attribute boundary vector matrix $\B{S} = {[\B{s}_1, \B{s}_2, \ldots, \B{s}_k]} \in \C{R}^{m \times k}$. The distance between the $i$-th semantic vector and the corresponding attribute boundary vector is defined as $ d_i = \|\B{w}_i - \B{s}_i\|_2^2$. The total distance, denoted by $d$, is computed as follows:
  {
  \begin{equation}
    \label{distanceSum}
    d=\sum_{i=1}^k d_i=\sum_{i=1}^k \| \B{w}_i-\B{s}_i \|_2^2=\|\B{W}-\B{S}\|_\text{F}^2.
  \end{equation}}

We incorporate the distance $d$ from Eq.~\eqref{distanceSum} as a regularization term in Eq.~\eqref{classicPCA} to encourage each semantic vector $\B{w}_i$ to move further away from the corresponding the attribute boundary vector $\B{s}_i$. Consequently, the objective of \MyMethod\ is reformulated as follows:
   
  \begin{equation}
  \label{objFun1}
  \begin{split}
    & \min_{\B{W},\B{P}} \|\B{Z}-\B{W}\B{P}\|_\text{F}^2- \lambda \|\B{W} - \B{S}\|_\text{F}^2,  \\
    & s.t \quad \B{W}^\top \B{W}=\B{I}, \B{W} \geq \B{0},
  \end{split}
  \end{equation}
where $\lambda>0$ is the regularization parameter. 
To minimize Eq.~\eqref{objFun1}, $-\| \B{W}-\B{S}\|_\text{F}^2$ should be as small as possible, encouraging the $i$-th semantic vector $\B{w}_i$ (i.e., the $i$-th column of $\B{W}$) to move further from the corresponding attribute boundary vector $\B{s}_i$ (i.e., the $i$-th column of $\B{S}$). This process enhances the disentanglement of the semantic vectors $\B{w}_i$.\par

The semantic vectors $\B{w}_i$ obtained from Eq.~\eqref{objFun1} allow for precise control over the magnitude of facial attribute manipulation. Specifically, for a given attribute editing magnitude $\beta$, the following cases are considered: (1) When $\beta > 1$, $\| \beta \B{w}_i - \B{s}_i \|_2 > \| \B{w}_i -\B{s}_i \|_2$, indicating that $\beta \B{w}_i$ moves further from $\B{s}_i$ than $\B{w}_i$, resulting in more significant editing results. (2) Similarly, when $0 < \beta < 1$, the editing effect is less significant, as $\beta \B{w}_i$ moves closer to $\B{s}_i$. (3) When $\beta < 0$, the editing effect of $\beta \B{w}_i$ is opposite to that of $\B{w}_i$. For example, if $\B{w}_i$ transforms a facial image into a smile, $\beta \B{w}_i (\beta < 0)$ generates a facial image with a crying expression~\citep{InterFaceGAN2020}. Therefore, the semantic vector $\B{w}_i$ facilitates precise control of facial attribute changes by adjusting the editing magnitude $\beta$.\par
In summary, the semantic vectors obtained by solving Eq.~\eqref{objFun1} offer the following advantages:
\begin{itemize}
  \item The orthogonal non-negativity constraint in the objective of \MyMethod\ promotes the sparsity of semantic vectors, further mitigating the entanglement issue.
  \item The semantic vectors for facial attribute editing that are distant from the attribute boundary vectors not only mitigate entanglement issues but also improve controllability.
\end{itemize}

\subsection{\FullMyAlgorithm\ Algorithm }
While introducing orthogonal and non-negativity constraints (Eq.~\eqref{objFun2}) substantially improves disentanglement and controllability, it also makes the optimization problem non-trivial. Existing off-the-shelf solvers are either computationally expensive or provide no convergence guarantee. To this end, we design a lightweight alternating iterative algorithm, termed \FullMyAlgorithm\ (\MyAlgorithm), which exploits closed-form updates for each subproblem. This design not only ensures provable convergence but also achieves significantly lower computational overhead compared to general-purpose optimization methods, making the framework practical for large-scale and {real-time} facial editing.
 Specifically, we introduce an auxiliary variable $\B{F}$ to transfer the non-negativity constraint $\B{W} \geq 0$ to $\B{F} \geq 0$. Subsequently, $\B{W}$ is aligned with $\B{F}$ by incorporating a regularization term into Eq.~\eqref{objFun1} as follows:
  \begin{equation}
    \label{objFun3}
    \begin{split}
      &\displaystyle \min_{\B{W},\B{P},\B{F}} \|\B{Z}-\B{W}\B{P}\|_\text{F}^2+\alpha\|\B{W}-\B{F}\|_\text{F}^2-\lambda \|\B{W}-\B{S}\|_\text{F}^2, \\
      & s.t \quad \B{W}^\top\B{W}=\B{I},\B{F}\geq \B{0},
    \end{split}
  \end{equation}
where $\alpha>0$ represents a regularization parameter that controls the degree of equivalence between $\B{F}$ and $\B{W}$. A larger $\alpha $ encourages $\B{F}$ closer to $\B{W}$, and vice versa.\par

The alternating iterative algorithm \MyAlgorithm\ is described below.\par
\subsubsection{Updating $\B{W}$ with $\B{P}$ and $\B{F}$ Fixed}
When $\B{P}$ and $\B{F}$ are fixed, Eq.~\eqref{objFun3} depends solely on $\B{W}$ and can be expressed as:
  {
  \begin{equation}
      \label{updateW}
      \begin{split}
        &\displaystyle \min_{\B{W}} \|\B{Z}-\B{W}\B{P}\|_\text{F}^2+\alpha\|\B{W}-\B{F}\|_\text{F}^2-\lambda \|\B{W}-\B{S}\|_\text{F}^2, \\
        & s.t \quad \B{W}^\top\B{W}=\B{I}.
      \end{split}
  \end{equation}}

Using the Theorem~\ref{thm:equivalent} presented in~\ref{appendixA}, we obtain the equivalent form of Eq.~\eqref{updateW} as follows:

  {
  \begin{equation}
    \label{derive_W}
  \begin{split}
  & \min_{\B{W}} \|\B{W}-\left(\B{Z}\B{P}^\top+\alpha \B{F}-\lambda \B{S}\right) \|_\text{F}^2, \\
  & s.t \quad \B{W}^\top\B{W}=\B{I}. \\
  \end{split}
  \end{equation}}

Eq.~\eqref{derive_W} corresponds to an Orthogonal Procrustes Problem (OPP) and has a closed-form solution derived from Proposition 1 in~\cite{SOCFS_CVPR_2015}
  \begin{equation}
    \label{calculate_W}
    \B{W}=\B{U}\B{I}_{m,k}\B{V}^\top,
  \end{equation}
where $\B{U} \in \C{R}^{m\times m}$ and $\B{V}^\top\in \C{R}^{k\times k}$ are the left and right {singular vectors} of $ (\B{Z}\B{P}^\top+\alpha \B{F}-\lambda \B{S})\in \C{R}^{m\times k}$, respectively. Both matrices are column orthogonal and satisfy $\B{U}^\top\B{U} = \B{I}_{m\times m}$ and $\B{V}^\top\B{V} = \B{I}_{k\times k}$, where $\B{I}_{m\times m}$ and $\B{I}_{k\times k}$ are the identity matrices of dimensions $m\times m$ and $k\times k$, respectively. 

\subsubsection{Updating $\B{P}$ with $\B{W}$ and $\B{F}$ Fixed}
When $\B{W}$ and $\B{F}$ are fixed, the objective in Eq.~\eqref{objFun3} is reduced to:
  \begin{equation}
  \label{object_P}
  \begin{split}
  &\displaystyle \min_{\B{P}} \|\B{Z}-\B{W}\B{P}\|_\text{F}^2. \\
  \end{split}
  \end{equation}

By taking the derivative of Eq.~\eqref{object_P} with respect to $\B{P}$ and setting it to zero:
  {
  \begin{equation}
    \label{calculate_P0}
    \begin{split} 
      & \B{P}^\top \B{W}^\top \B{W} - (\B{W}^\top \B{Z} )^\top=0 \\
      & \Longrightarrow  \B{W}^\top \B{W} \B{P}=(\B{W}^\top \B{Z} ). \\
    \end{split}
  \end{equation}}

Since $\B{W}^\top \B{W} =\B{I}_{k\times k}$, the coefficient matrix $\B{P}$ in Eq.~\eqref{calculate_P0} admits the following closed-form solution:
  \begin{equation}
    \label{calculate_P}
    \B{P}=\B{W}^\top \B{Z}.
  \end{equation}

\subsubsection{Updating $\B{F}$ with $\B{W}$ and $\B{P}$ Fixed}
With fixed $\B{W}$ and $\B{P}$, the sub-problems for $\B{F}$ are formulated as follows:
   
  \begin{equation}
      \label{objFunForF}
    \begin{split}
      &\displaystyle \min_{\B{F}} \|\B{W}-\B{F}\|_\text{F}^2, \\
      & s.t \quad \B{F}\geq \B{0}.
    \end{split} 
  \end{equation}

 Eq.~\eqref{objFunForF} can be directly computed as~\cite{SOCFS_CVPR_2015}:
   
  \begin{equation}
  \label{calculate_F}
    \B{F}=\frac{1}{2}\left(\B{W}+|\B{W}|\right),
  \end{equation}
where $|\cdot|$ represents the element-wise absolute value of the matrix $\B{W}$.\par
The proposed \MyAlgorithm\ algorithm for solving Eq.~\eqref{objFun3} is detailed in Algorithm~\ref{alg_1}.
\begin{algorithm}[h]
  \caption{Alternating Iterative Disentanglement and Controllability (\MyAlgorithm) Algorithm}\label{alg_1}
  \setstretch{1.1}
  \begin{algorithmic}[1]
  \REQUIRE
  The latent vector matrix $\B{Z} \in \C{R}^{m\times k}$ and the attribute boundary vector matrix $\B{S} \in \C{R}^{m \times k}$, the number of iterations $n=30$, the regularization parameter $\alpha=0.5$ and $\lambda=1$.
  \ENSURE
  The semantic vectors matrix $\B{W}={[w_1, w_2, \ldots, w_k]} \in \C{R}^{m \times k} $.
  \STATE~\textbf{Initialization:} randomly initializing $\B{W}^1$ and $\B{F}^1$, $\B{P}^1=(\B{W}^1)^\top\B{Z}$
  \WHILE{$1 \leq t \leq n$}
  \STATE~When $\B{P}^{t}$ and $\B{F}^{t}$ are fixed, update $\B{W}^{t+1}=\B{U}\B{I}_{m,k}\B{V}^\top$ by Eq.~\eqref{calculate_W}, where $\B{U},\B{V}^\top$ are left and right eigenvectors of $(\B{Z}(\B{P}^t)^\top+\alpha \B{F}^t-\lambda \B{S})$.
  \STATE~When $\B{W}^{t+1}$ and $\B{F}^{t}$ are fixed, updating $\B{P}^{t+1}=(\B{W}^{t+1})^\top\B{Z}$ by Eq.~\eqref{calculate_P}.
  \STATE~When $\B{W}^{t+1}$ and $\B{P}^{t+1}$ are fixed, updating $\B{F}^{t+1}=\frac{1}{2}\left(\B{W}^{t+1}+|\B{W}^{t+1}|\right)$ by Eq.~\eqref{calculate_F}.
  \STATE~$t=t+1$
  \ENDWHILE
  \STATE~$\B{W}=\B{W}^{n+1}$
  \end{algorithmic}
\end{algorithm}
\subsection{Convergence and Time Complexity Analysis of \MyAlgorithm\ }

\subsubsection{Convergence Guarantees of \MyAlgorithm}
\label{subsec:aidc_convergence}
We define the objective as:
\begin{equation}
\label{JWEq}
J(\B{W},\B{P},\B{F})=\|\B{Z}-\B{W}\B{P}\|_\text{F}^2+\alpha\|\B{W}-\B{F}\|_\text{F}^2-\lambda\|\B{W}-\B{S}\|_\text{F}^2.
\end{equation}\par
{We analyze the convergence of \MyAlgorithm\ for the objective $J$ in Eq.~\eqref{JWEq} under the constraints
$\B{W}^\top\B{W}=\B{I}$ and $\B{F}\ge \B{0}$. Each iteration performs exact block minimization over $\B{W}$, $\B{P}$, and $\B{F}$, and each block in \MyAlgorithm\ admits an exact closed-form minimizer:
(i) the $\B{W}$-update is an orthogonal procrustes problem solved by the thin SVD,
(ii) $\B{P}=\B{W}^\top \B{Z}$ (least squares with $\B{W}^\top\B{W}=\B{I}$),
and (iii) $\B{F}=\max(\B{W},\B{0})=\tfrac12\bigl(\B{W}+|\B{W}|\bigr)$ (Euclidean projection onto the nonnegative orthogonality).
These exact block updates imply a non-increasing objective and lead to the following result.}

\begin{thm}[Monotone descent and first-order stationarity]\label{thm:Convergence}
  {Let $\{(\B{W}^t,\B{P}^t,\B{F}^t)\}$ be the sequence generated by \MyAlgorithm\ with $\alpha>0$ and $\lambda\ge 0$. Then the objective values are 
  non-increasing:}
\begin{equation}
{J(\B{W}^{t+1},\B{P}^{t+1},\B{F}^{t+1}) \le J(\B{W}^{t},\B{P}^{t},\B{F}^{t})\quad\text{for all } t,}
\end{equation}
{and the sequence $J(\B{W}^{t},\B{P}^{t},\B{F}^{t})$ converges.
Moreover, every accumulation point is a first-order stationary point in Eq.~\eqref{JWEq}}.\par
\textit{Proof sketch.}
{Exact minimization of each block yields $J(\B{W}^{t+1},\B{P}^{t+1},\B{F}^{t+1}) \le J(\B{W}^{t},\B{P}^{t},\B{F}^{t})$, hence non-increasing sequence of objective values.
Since $\mathbf W$ lies on the Stiefel manifold (a compact set), $\alpha>0$ makes $\alpha\|\mathbf W-\mathbf F\|_F^2$ coercive in $\mathbf F$, and the least-squares update keeps $\mathbf P$ bounded, the objective $J$ is lower bounded. Hence $J(\cdot)$ converges. Full details are provided in~\ref{appendixB}.} {Note that the objective function is non-convex due to the nonlinear constraints and the problem structure. Therefore, we cannot guarantee the existence of a global optimum for the objective function. Our analysis shows that the objective value $J(\cdot)$ converges to a finite limit, and every accumulation point of the iterative sequence is a first-order stationary point (e.g., satisfies the KKT conditions). We emphasize that, for non-convex problems, the convergence of the iterations themselves is generally not guaranteed; the sequence may oscillate, and only subsequential convergence can be ensured under the stated assumptions.}
\end{thm}

\subsubsection{Time Complexity Analysis}
The time complexity of \MyAlgorithm\ comprises the following three parts:
\begin{itemize}
  \item When $\B{P}$ and $\B{F}$ are fixed, updating $\B{W}$ using Eq.~\eqref{calculate_W} requires performing Singular Value Decomposition (SVD) on matrix $ (\B{Z}\B{P}^\top+\alpha \B{F}-\lambda \B{S})\in \C{R}^{m\times k}$, where $m$ represents the dimension of the latent vectors, typically set to $512$, and $k$ denotes the number of semantic vectors. Since $m \gg k$, the time complexity of SVD is $O(m^2k)$. The computation of $\B{W}=\B{U}\B{I}_{m,k}\B{V}^\top$ involves three matrix multiplications, where $\B{U} \in \mathbb{R}^{m\times m}$, $\B{V}^\top\in \mathbb{R}^{k\times k}$ and $\B{I}_{m,k} \in \mathbb{R}^{m\times k}$ is a diagonal matrix. Therefore, the time complexity is $O(mk^2)$. Thus, the overall time complexity of this update step is $O(m^2k)+O(mk^2)$.
  \item When $\B{W}$ and $\B{F}$ are fixed, updating $\B{P}$ using Eq.~\eqref{calculate_P}, where $\B{Z} \in \C{R}^{m\times k}$ and $\B{W} \in\C{R}^{m\times k}$, results in an overall time complexity $O(mk^2)$.
  \item When $\B{W}$ and $\B{P}$ are fixed, updating $\B{F}$ using Eq.~\eqref{calculate_F} results in an overall time complexity $O(mk)$.
\end{itemize}

In summary, the overall time complexity of \MyAlgorithm\ is $O(m^2k)+O(mk^2)$. Given $m \gg k$, the complexity of the \MyAlgorithm\ approximates $O(m^2)$. The computational cost of \MyAlgorithm\ primarily depends on the dimension $m$ of the latent vector, which is typically fixed at $512$. This makes \MyAlgorithm\ both efficient and scalable, as it is relatively insensitive to the number of sampled latent vectors. Consequently, the proposed framework \MyMethod\ is well-suited for large-scale and {real-time} applications.
\section{Experimental Validation}
~\label{sec:exp} 
To comprehensively validate the effectiveness of \MyMethod, we design a systematic set of experiments encompassing both objective and subjective evaluations. The validation is organized into four aspects: (1) disentanglement evaluation, (2) controllability of attribute manipulation, (3) image quality assessment, (4) in-depth discussions, including identity consistency, parameter sensitivity analysis, user study and visualization of challenging cases. This holistic design ensures that both the technical soundness and the practical reliability of \MyMethod\ are rigorously assessed against state-of-the-art supervised, unsupervised, and diffusion-based baselines.

\subsection{Experimental Details and Reproducibility} 
\label{sec:exp_details}
{This section details the hardware/software environment, latent spaces, sampling protocol, initialization and update rules, hyperparameter selection, and evaluation pipeline to facilitate reproducibility.}

{\textbf{Hardware \& Software.}}
\begin{itemize}
    \item {\textbf{GPU:} NVIDIA A100 80GB (PCIe)}; {\textbf{CPU:} Intel Xeon Platinum 8358}
    \item {\textbf{OS:} Ubuntu 18.04.1 LTS}; {\textbf{Framework:} PyTorch~1.10 {(CUDA~11.3, cuDNN~8.2)}}
    \item {\textbf{Determinism:} Seeds are fixed and deterministic execution is enabled (cuDNN non-deterministic algorithms turned off).}
\end{itemize}

{\textbf{Latent space and sampling.}}
{We evaluate three pre-trained GAN backbones: StyleGAN~\citep{StyleGAN2019}, StyleGAN2~\citep{Karras.2020}, and ProGAN~\citep{ProGAN2018}. Among them, StyleGAN and StyleGAN2 are pre-trained on the FFHQ dataset, while ProGAN is pre-trained on the CelebA-HQ dataset.}
{For StyleGAN/StyleGAN2 we operate in the \emph{W} space (512-d; truncation $\psi{=}1.0$); for ProGAN we use its 512-d latent space. We sample $n{=}10{,}000$ latent vectors to form $\B{Z}\!\in\!\mathbb{R}^{m\times n}$ with $m{=}512$. The semantic vector matrix is $\B{W}\!\in\!\mathbb{R}^{m\times k}$ with $k{=}5$ (Age, Gender, Hairstyle, Pose, Smile), and coefficients $\B{P}\!\in\!\mathbb{R}^{k\times n}$.}\par
{\textbf{Attribute boundaries.} We adopt semantic boundaries for the five attributes.}
{Each boundary vector $\B{s}_i$ is $\ell_2$-normalized before forming $\B{S}=[\B{s}_1,\ldots,\B{s}_5]$. The regularizer weight $\lambda$ controls the influence of $\B{S}$; setting $\lambda{=}0$ reduces to our orthogonal-nonnegative subspace baseline.}

{\textbf{Initialization and updates (Algorithm~\ref{alg_1}).} {We follow Algorithm.~\ref{alg_1}.  $\B{W}^{(1)}$ is initialized by the top-$k$ thin-SVD left singular vectors of $\B{Z}$; $\B{F}^{(1)}=\tfrac12(\B{W}^{(1)}+|\B{W}^{(1)}|)$; $\B{P}^{(1)}=(\B{W}^{(1)})^\top\B{Z}$. Each iteration solves the OPP via thin SVD and sets $\B{W}=\B{U}\B{V}^\top$. We stop after $30$ iterations or when the relative objective drops $<10^{-6}$.}}

{\textbf{Hyperparameters and selection.}}
{Unless otherwise stated, we set $\alpha{=}0.5$ and $\lambda{=}1.0$. Values are selected via 5-fold cross-validation on the latent list, optimizing the average of FID and LPIPS; seeds are fixed across folds. For the controllability evaluation (Section~\ref{sec:controllability}), the editing range for the Age, Smile, and Hairstyle attributes is set as $\beta \in \{-0.3, -0.2, 0.2, 0.3\}$, while for the Gender attribute, $\beta$ takes values in $\{-0.3, -0.1, 0.1, 0.3\}$. For identity consistency analysis (\ref{sec:identity_consistency}), we set $\beta$ values within the range of [0.1, 0.8] and observe the changes in identity consistency.}

{\textbf{Evaluation pipeline.}}
{For each attribute $k$, we generate $100$ original/edited pairs
$\{M_{\text{ori}}^j(k),\,M_{\text{edit}}^j(k)\}_{j=1}^{100}$.
Images are center-cropped and resized to $256$ px for consistency.
All reported numbers are the mean of three seeded runs.}
{Definitions and implementations of the evaluation metrics are provided in Section~\ref{sec:ME}; the disentanglement correlation is computed as in Eq.~\eqref{corrcompute}.}

{\textbf{Fairness controls.}}
{To ensure a fair comparison with baselines (see Section~\ref{sec:BaseMethod}), all methods are run on the same A100-80GB under identical software/parameter settings and a unified evaluation protocol. GAN baselines follow authors' defaults; diffusion baselines use authors' code with their recommended sampling steps, guidance scales, and resolution. Latency is measured over $100$ images and excludes model initialization time.}

\subsection{Metrics for Evaluation}
\label{sec:ME}
To quantitatively evaluate the disentanglement achieved by the proposed \MyMethod, we calculate the Pearson correlation coefficient~\citep{MDSE2023} between the $i$-th and $k$-th attributes for the pre-trained StyleGAN2. A lower correlation coefficient indicates better disentanglement performance.\par
We employ four widely used metrics in facial attribute editing to evaluate the quality of the edited facial images: (1) Fréchet Inception Distance (FID)~\citep{NIPS2017FID}, (2) Learned Perceptual Image Patch Similarity (LPIPS)~\citep{CVPR2018LPIPS}, (3) Structural Similarity Index (SSIM)~\citep{TIP2004SSIM}, and (4) Inception Score (IS)~\citep{NIP2016IS}. \par
The values of LPIPS and SSIM range from $0$ to $1$, whereas the FID and IS do not have fixed ranges. Higher SSIM and IS values indicate greater structural similarity between images, along with clearer and more diverse outputs. Conversely, lower LPIPS and FID values suggest that the edited images align more closely with the original images in terms of feature distribution and visual perception, reflecting more natural and high-quality editing effects. In our experiments, the best metric results are highlighted in bold, while the second-best results are underlined.\par

\subsection{Baseline Methods}
\label{sec:BaseMethod}
We compare \MyMethod\ against ten state-of-the-art facial attribute editing methods, including three supervised GAN-based methods (\InterFaceGAN~\citep{InterFaceGAN2020}, \enjoyGAN~\citep{EnjoyGAN2021}, and \MDSE~\citep{MDSE2023}), four unsupervised GAN-based methods (\GANspace~\citep{GANSpace2020}, \sefa~\citep{SEFA2021}, \AdaTans~\citep{AdaTrans2023}, and \SDFlow~\citep{SDFlow2024ICASSP}), and three diffusion-based methods (DDS~\citep{dds2023}, CDS~\citep{cds2024}, and FPE~\citep{FPE2024}). All baseline methods are implemented using the default model configurations released in their respective papers. Moreover, \MyMethod\ shares methodological similarities with \GANspace, \InterFaceGAN, \sefa, and \MDSE\ in the way semantic vectors are computed.

\subsection{Disentanglement Evaluation}
In this section, we evaluate the disentanglement performance of the semantic vectors computed by \MyMethod, through both quantitative and qualitative evaluations on pre-trained GAN models and diffusion models.\par

\subsubsection{Qualitative Evaluation of \MyMethod's Disentanglement}
{We qualitatively compare the disentanglement performance of \MyMethod\ with \sefa, \InterFaceGAN, \GANspace, \enjoyGAN, \MDSE, and \AdaTans\ on StyleGAN2 and StyleGAN, focusing on Age, Gender, Hairstyle, and Smile attributes. The results are shown in Fig.~\ref{fig:stylegan2}, Fig.~\ref{fig:stylegan}, and Fig.~\ref{fig:diffusion}. For the Age attribute (Fig.~\ref{fig:stylegan2}, first row), only \MyMethod\ edits the attribute without affecting others, while other methods unintentionally alter non-target attributes (e.g., \sefa\ adds a hat or \enjoyGAN\ removes forehead decorations). A similar pattern is seen for Gender and Hairstyle edits (Fig.~\ref{fig:stylegan2}, second and third rows), where \MyMethod\ isolates the target attribute, while others alter skin tone, nose size, or add a smile. When editing the Smile attribute (Fig.~\ref{fig:stylegan2}, last row), only \MyMethod\ and \AdaTans\ successfully add a smile without affecting other facial features.}

{Results on StyleGAN (Fig.~\ref{fig:stylegan}) are consistent with those on StyleGAN2, further confirming \MyMethod's superior disentanglement. For example, while \sefa\ and \InterFaceGAN\ can edit Gender effectively, they distort facial features, causing identity inconsistencies. Similarly, \GANspace\ inadvertently removes eyeglasses while editing Hairstyle. We also compare \MyMethod\ with several diffusion-based methods (Fig.~\ref{fig:diffusion}). While diffusion-based methods reduce entanglement, methods like DDS, CDS, and FPE often produce less realistic edits. In contrast, \MyMethod\ maintains higher perceptual fidelity and more natural results. Overall, GAN-based methods suffer from attribute entanglement, whereas diffusion-based methods, though reducing entanglement, produce less realistic edits. \MyMethod\ effectively mitigates entanglement while preserving the naturalness of the original image.}

\begin{figure*}[h]
  \centering
  \begin{subfigure}{0.49\linewidth}
    \centering
    \includegraphics[width=\linewidth]{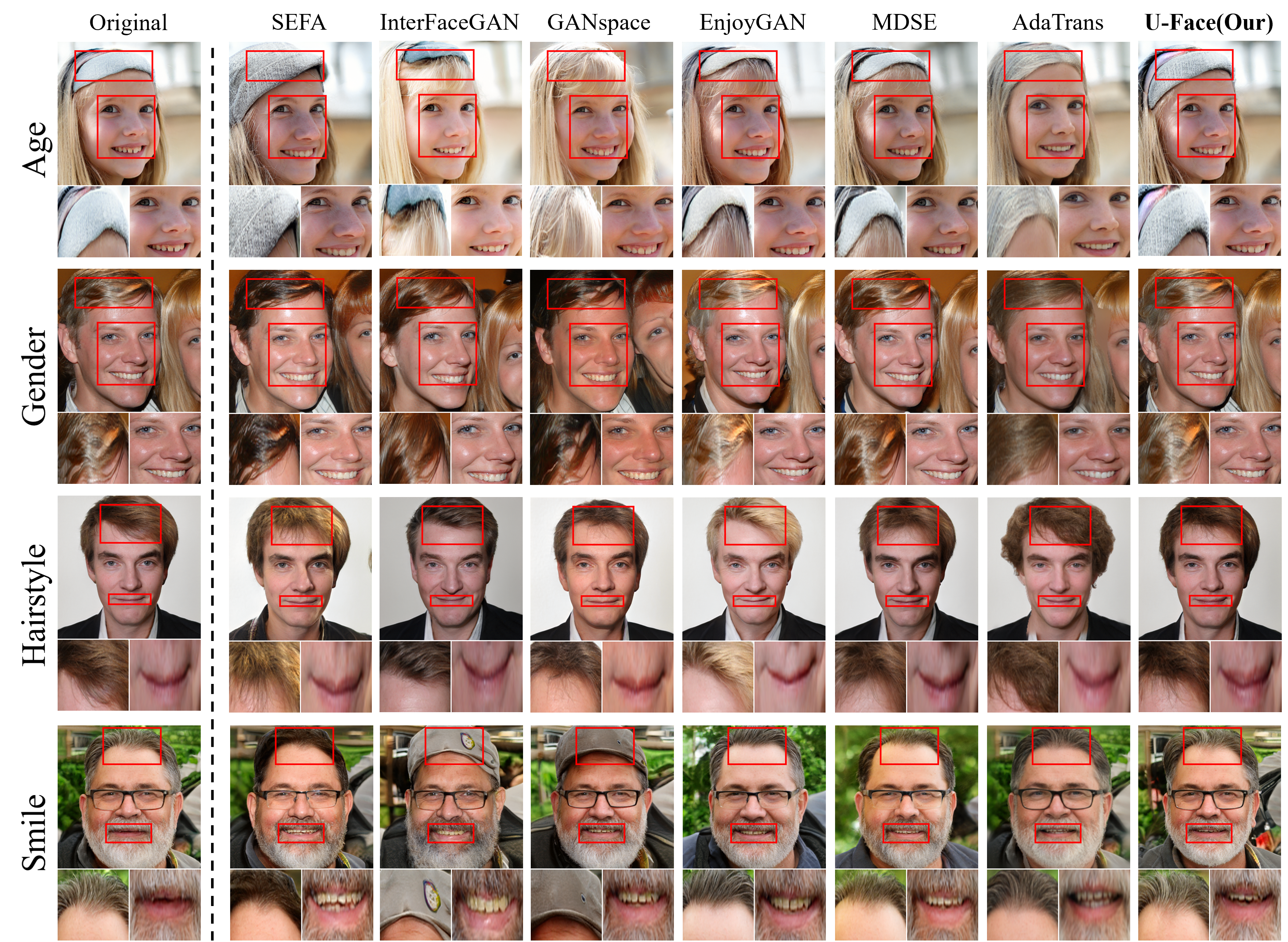}
    \caption{Result on StyleGAN2}
    \label{fig:stylegan2}
  \end{subfigure}
  \hfill
  \begin{subfigure}{0.49\linewidth}
    \centering
    \includegraphics[width=\linewidth]{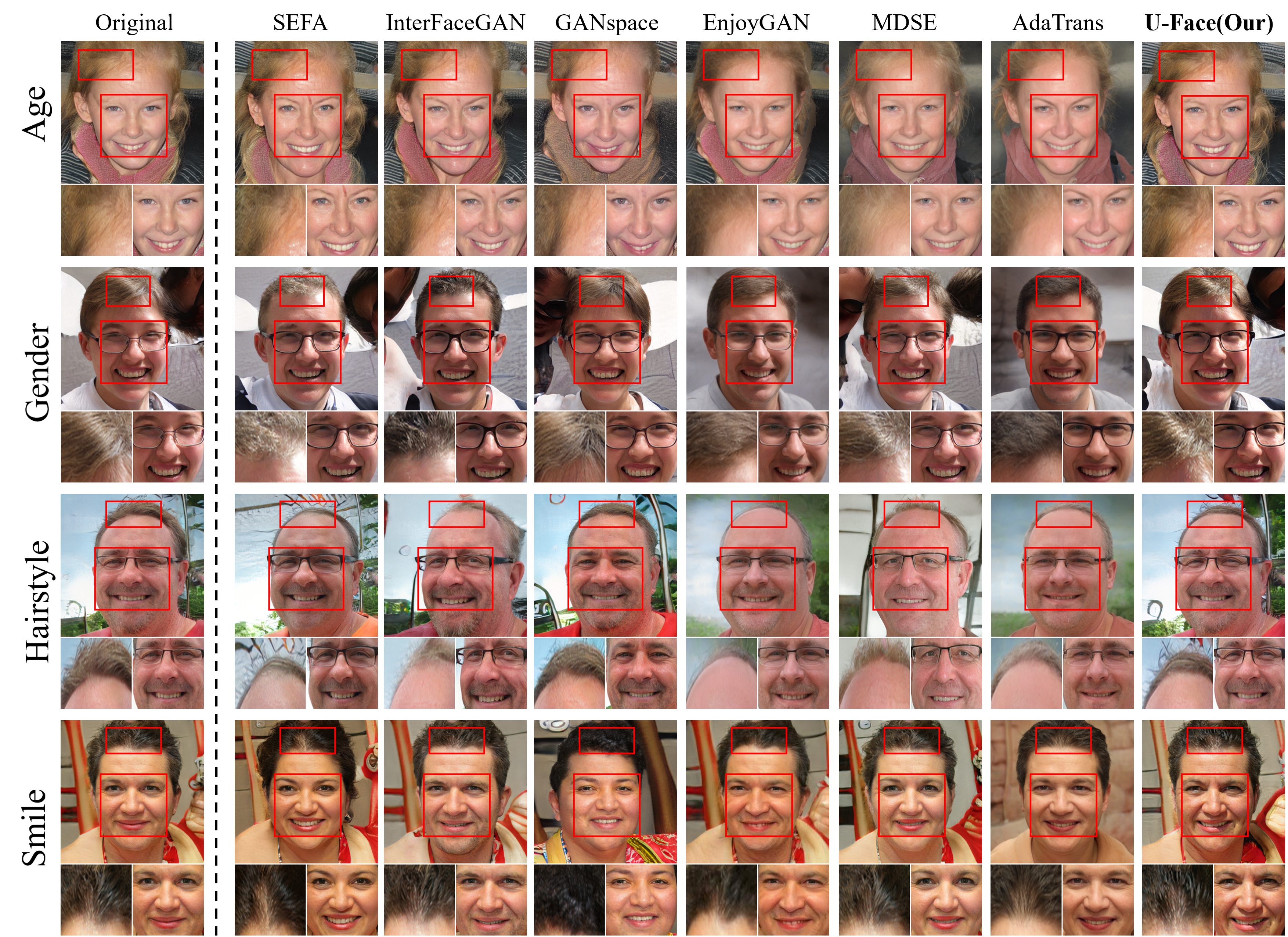}
    \caption{Result on StyleGAN.}
    \label{fig:stylegan}
  \end{subfigure}
  \caption{Qualitative comparison of \MyMethod\ and GAN-based baseline methods on StyleGAN and StyleGAN2 across the Age, Gender, Hairstyle, and Smile attributes.}
  \label{fig:stylegan_stylegan2}
\end{figure*}
\begin{figure*}[h]
  \centering 
  \includegraphics[width=0.5\linewidth]{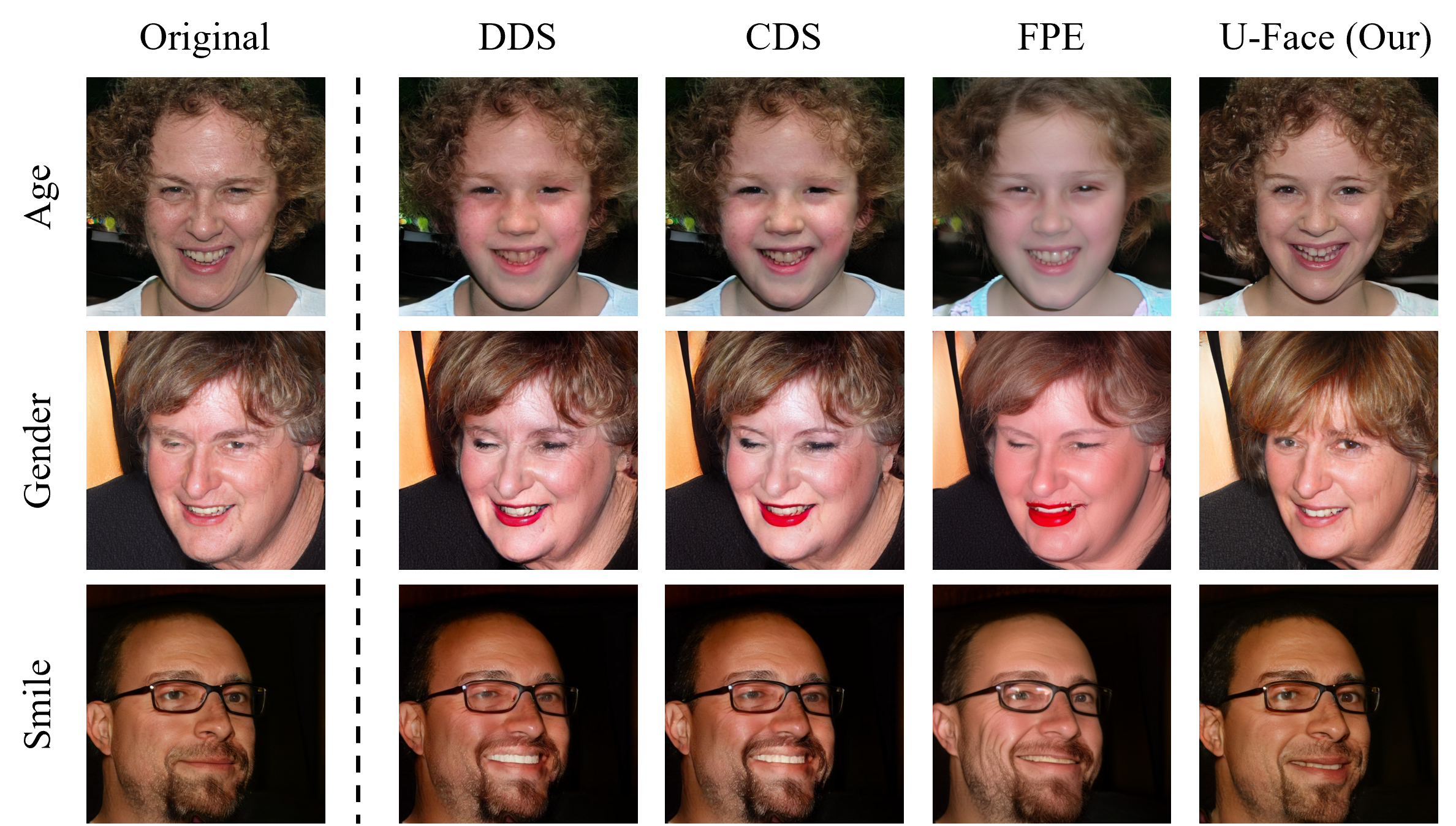}
  \caption{Qualitative evaluation of \MyMethod\ and diffusion-based baseline methods across on the Age, Gender, and Smile attributes.}
  \label{fig:diffusion}
\end{figure*}

\subsubsection{Quantitative Evaluation of \MyMethod's Disentanglement}
{We evaluate the disentanglement performance of semantic vectors computed by \MyMethod\ using the Pearson correlation coefficient, compared to \InterFaceGAN, \sefa, and \MDSE\ (details in~\ref{sec:corrcompute}). The results in Tables~\ref{tab:disentanglement1} and~\ref{tab:disentanglement2} show that \MyMethod\ achieves lower correlation coefficients across most attribute pairs, indicating better disentanglement. Specifically, \MyMethod\ outperforms \InterFaceGAN\ by $2\%$ to $23\%$ (except for gender-age), and outperforms \MDSE\ by $7\%$ to $10\%$ (except for pose-smile and pose-age). It also shows a $2\%$ to $40\%$ improvement over \GANspace\ and, although \sefa\ shows improvement for age-smile, it generally has higher correlations than \MyMethod\ for other pairs, with increases of $9\%$ to $22\%$.}

{We also compute the average Pearson correlation (excluding the diagonal) for each method. \MyMethod\ significantly reduces the correlation by $10\%$ to $15\%$ compared to \InterFaceGAN\ and \sefa, $3\%$ to $5\%$ compared to \MDSE, and $11\%$ to $36\%$ compared to \GANspace. In summary, \MyMethod\ effectively addresses semantic vector entanglement through attribute boundary vectors and orthogonal non-negative constraints, leading to high-quality facial attribute editing with fewer unintended edits and a lower Pearson correlation coefficient.}

\begin{table*}[!htb]
  \centering
  \caption{Correlation matrix for \InterFaceGAN, \MDSE\ and \MyMethod\ across different facial attributes. Standard deviations are shown in parentheses. \textbf{Avg} represents the average value of each column, excluding diagonal values.} 
  \label{tab:disentanglement1}
  \renewcommand{\arraystretch}{1.8}
  \setlength{\tabcolsep}{6pt} 
  \resizebox{\textwidth}{!}{
    \begin{tabular}{c|ccccc|ccccc|ccccc}
      \toprule
      \hline
      {Method} & \multicolumn{5}{c|}{\makecell{\InterFaceGAN\ \citep{InterFaceGAN2020}}} & \multicolumn{5}{c|}{\makecell{\MDSE\ \citep{MDSE2023}}} & \multicolumn{5}{c}{\textbf{\MyMethod\ (Our)}} \\    
      \cline{2-16}
      & Age & Gender & Hairstyle & Pose & Smile & Age & Gender & Hairstyle & Pose & Smile & Age & Gender & Hairstyle & Pose & Smile \\ 
      \midrule
      Age & {--} & 0.565 (0.01) & 0.232 (0.01) & 0.090 (0.00) & 0.154 (0.02) & {--} & 0.388 (0.03) & 0.265 (0.01) & \textbf{0.008 (0.00)} & 0.120 (0.01) & {--} & \textbf{0.311 (0.02)} & \textbf{0.211 (0.03)} & \underline{0.037 (0.00)} & \underline{0.017 (0.00)} \\
      Gender & 0.565 (0.01) & {--} & 0.268 (0.01) & \textbf{0.017 (0.00)} & 0.198 (0.02) & 0.388 (0.03) & {--} & 0.212 (0.00) & \underline{0.029 (0.00)} & 0.050 (0.00) & \textbf{0.311 (0.02)} & {--} & \textbf{0.191 (0.01)} & 0.032 (0.00) & \textbf{0.015 (0.00)} \\
      Hairstyle & 0.232 (0.01) & 0.268 (0.01) & {--} & 0.302 (0.02) & 0.254 (0.02) & 0.265 (0.01) & 0.212 (0.00) & {--} & \underline{0.156 (0.01)} & 0.033 (0.00) & \textbf{0.211 (0.03)} & \textbf{0.191 (0.01)} & {--} & \textbf{0.111 (0.01)} & \textbf{0.015 (0.00)} \\
      Pose & 0.090 (0.00) & \textbf{0.017 (0.00)} & 0.302 (0.02) & {--} & 0.098 (0.01) & \textbf{0.008 (0.00)} & \underline{0.029 (0.00)} & \underline{0.156 (0.01)} & {--} & \textbf{0.018 (0.00)} & \underline{0.037 (0.00)} & 0.032 (0.00) & \textbf{0.111 (0.01)} & {--} & \underline{0.025 (0.00)} \\
          Smile & 0.154 (0.02) & 0.198 (0.02) & 0.254 (0.02) & 0.098 (0.01) & {--} & 0.120 (0.01) & \underline{0.050 (0.00)} & \underline{0.033 (0.00)} & \textbf{0.018 (0.00)} & {--} & \underline{0.017 (0.00)} & \textbf{0.015 (0.00)} & \textbf{0.015 (0.00)} & \underline{0.025 (0.00)} & {--} \\
          \midrule
          \textbf{Avg.} & 0.260 (0.01) & 0.262 (0.01) & 0.264 (0.01) & 0.127 (0.01) & 0.176 (0.02) & \underline{0.195 (0.01)} & \underline{0.170 (0.01)} & \underline{0.167 (0.01)} & \underline{0.053 (0.00)} & \underline{0.055 (0.00)} & \textbf{0.144 (0.01)} & \textbf{0.137 (0.01)} & \textbf{0.132 (0.01)} & \textbf{0.051 (0.00)} & \textbf{0.018 (0.00)} \\
      \hline
      \bottomrule
    \end{tabular}
  }
\end{table*}

\begin{table*}[!htb]
  \centering
  \caption{Correlation matrix for \GANspace, \sefa\ and \MyMethod\ across different facial attributes. Standard deviations are shown in parentheses.~\textbf{Avg} represents the average value of each column, excluding diagonal values.}
 ~\label{tab:disentanglement2}
  \renewcommand{\arraystretch}{1.8}
  \setlength{\tabcolsep}{6pt} 
  \resizebox{\textwidth}{!}{
    \begin{tabular}{c|ccccc|ccccc|ccccc}
      \toprule
      \hline
      {Method} & \multicolumn{5}{c|}{\makecell{\GANspace\ \citep{GANSpace2020}}} & \multicolumn{5}{c|}{\makecell{\sefa\ \citep{SEFA2021}}} & \multicolumn{5}{c}{\textbf{\MyMethod (Our)}} \\    
      \cline{2-16}
      & Age & Gender & Hairstyle & Pose & Smile & Age & Gender & Hairstyle & Pose & Smile & Age & Gender & Hairstyle & Pose & Smile \\ 
      \midrule
      Age & {--} & 0.715 (0.01) & 0.309 (0.01) & 0.116 (0.01) & 0.125 (0.01) & {--} & 0.533 (0.01) & 0.366 (0.01) & 0.039 (0.01) & \textbf{0.001 (0.00)} & {--} & \textbf{0.311 (0.02)} & \textbf{0.211 (0.03)} & \underline{0.037 (0.00)} & \underline{0.017 (0.00)} \\
      Gender & 0.715 (0.01) & {--} & 0.546 (0.01) & 0.177 (0.01) & 0.037 (0.00) & 0.533 (0.01) & {--} & 0.281 (0.01) & 0.133 (0.01) & 0.170 (0.01) & \textbf{0.311 (0.02)} & {--} & \textbf{0.191 (0.01)} & 0.032 (0.00) & \textbf{0.015 (0.00)} \\
      Hairstyle & 0.309 (0.01) & 0.546 (0.01) & {--} & 0.234 (0.01) & 0.202 (0.01) & 0.366 (0.01) & 0.281 (0.01) & {--} & 0.211 (0.01) & 0.180 (0.01) & \textbf{0.211 (0.03)} & \textbf{0.191 (0.01)} & {--} & \textbf{0.111 (0.01)} & \textbf{0.015 (0.00)} \\
      Pose & 0.116 (0.01) & 0.177 (0.01) & 0.234 (0.01) & {--} & 0.153 (0.01) & 0.039 (0.01) & 0.133 (0.01) & 0.211 (0.01) & {--} & 0.150 (0.01) & \underline{0.037 (0.00)} & 0.032 (0.00) & \textbf{0.111 (0.01)} & {--} & \underline{0.025 (0.00)} \\
      Smile & 0.125 (0.01) & 0.037 (0.00) & 0.202 (0.01) & 0.153 (0.01) & {--} & \textbf{0.001 (0.00)} & 0.170 (0.01) & 0.180 (0.01) & 0.150 (0.01) & {--} & \underline{0.017 (0.00)} & \textbf{0.015 (0.00)} & \textbf{0.015 (0.00)} & \underline{0.025 (0.00)} & {--} \\
      \midrule
      \textbf{Avg.} & 0.316 (0.01) & 0.369 (0.01) & 0.323 (0.01) & 0.170 (0.01) & 0.129 (0.01) & 0.235 (0.01) & 0.279 (0.01) & 0.260 (0.01) & 0.133 (0.01) & 0.125 (0.01) & \textbf{0.144 (0.01)} & \textbf{0.137 (0.01)} & \textbf{0.132 (0.01)} & \textbf{0.051 (0.00)} & \textbf{0.018 (0.00)} \\
      \hline
      \bottomrule
    \end{tabular}
  }
\end{table*}

\subsection{Controllability Evaluation}
\label{sec:controllability}
This section evaluates the controllability of facial attribute editing using semantic vectors computed by \MyMethod\ for varying editing magnitudes, focusing on Age, Smile, Hairstyle, and Gender attributes. As shown in Fig.~\ref{fig:combined}, the semantic vectors for editing these attributes exhibit precise controllability. {For the Age attribute, increasing the editing magnitude from $-0.3$ to $0.3$ adjusts the age of the edited images from approximately $10$ to $50$ years, as estimated using the model in~\cite{AgeEstimationCVPR2019}.} The second and third rows of Fig.~\ref{fig:combined} show the controllability of Smile and Hairstyle attribute editing, with results similar to Age editing. {For the Smile attribute, increasing the negative $\beta$ gradually diminishes the smile until it disappears, while increasing $\beta$ in the positive direction enhances the smile's intensity.} Hair volume is the primary indicator for Hairstyle editing, with negative $\beta$ values reducing hair volume and positive values increasing it gradually.\par
\begin{figure*}[h]
  \centering
  \includegraphics[width=\linewidth]{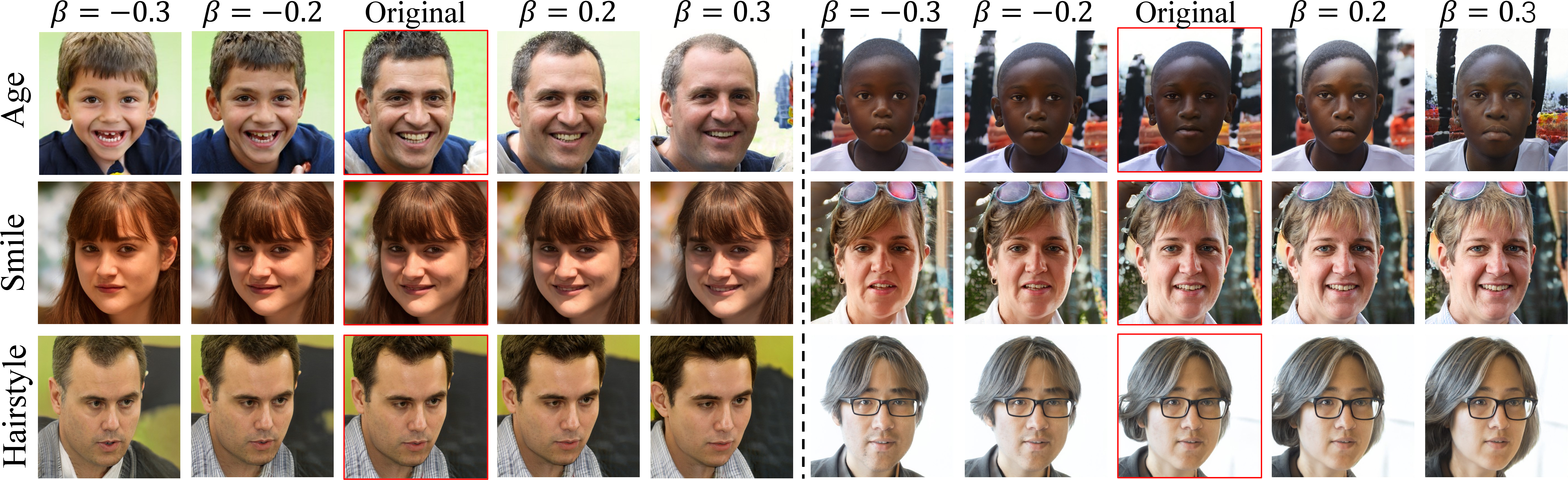}
  \vspace{-10pt}
  \caption{Qualitative evaluation of \MyMethod\ in controllable editing for Age, Smile, and Hairstyle attributes. When the editing magnitude $\beta$ is positive (e.g., $\beta = 0.3$) or negative (e.g., $\beta = -0.3$), the edited images exhibit significantly changes.}
~\label{fig:combined}
\end{figure*}
\begin{figure*}[h]
  \centering
  \includegraphics[width=0.6\linewidth]{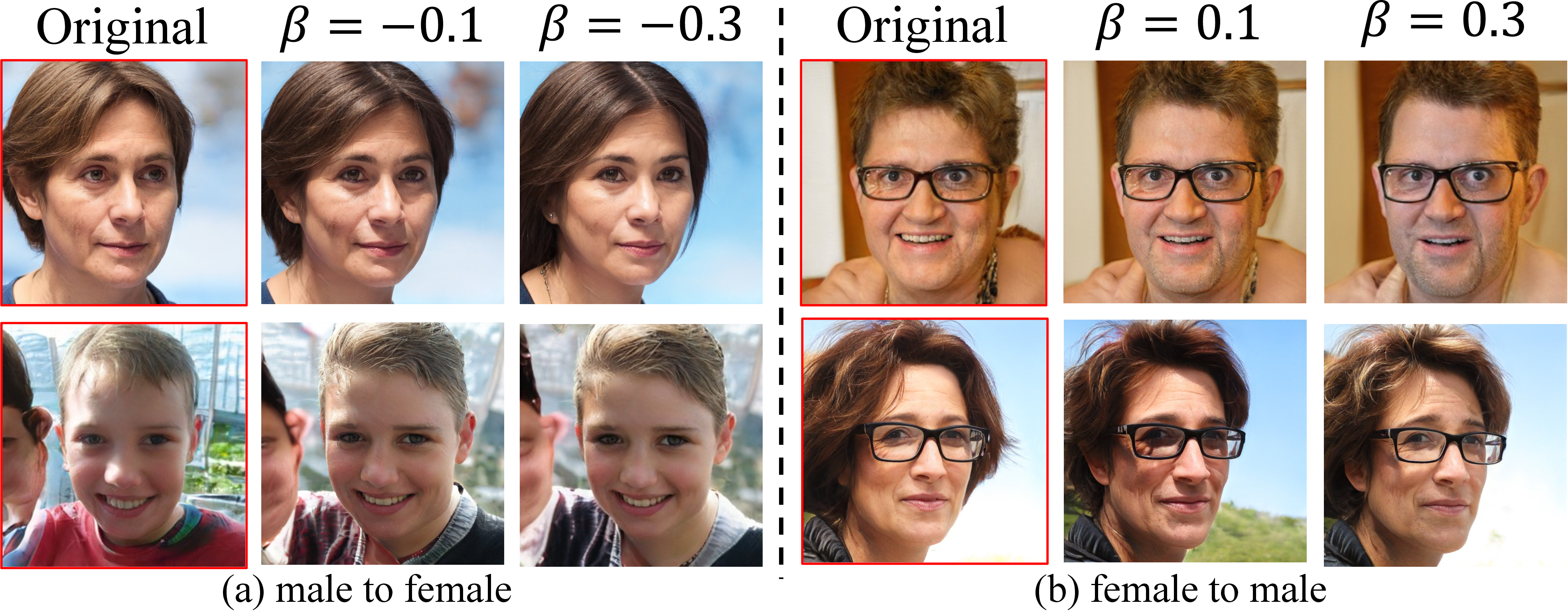}
  \caption{Qualitative comparison of \MyMethod\ in controllable editing for the Gender attribute. On the left is the transformation from male to female and on the right is the transformation from female to male. When the editing magnitude increases, the change in the Gender attribute becomes more significant.}
~\label{fig:gender}
\end{figure*}
For Gender attribute editing, transformations between males and females are shown in Fig.~\ref{fig:gender}. {When the editing magnitude $\beta$ is negative, the transformation shifts from male to female (Fig.~\ref{fig:gender}a), and increasing the magnitude results in more pronounced gender changes, while decreasing it reverses the effect.} The results also reveal slight entanglement with the Hairstyle attribute during Gender editing, due to the natural correlation between Gender and Hairstyle. However, no entanglement is observed with the Smile, Age, or Pose attributes. In summary, \MyMethod\ accurately adjusts facial attributes based on varying editing magnitudes.\par

\subsection{Quantitative Evaluation for Edited Image Quality}
{This section evaluates the quality of face images edited with semantic vectors from the \MyMethod\ framework, using FID, IS, LPIPS, and SSIM metrics. As shown in Table~\ref{tab:FID}, \MyMethod\ outperforms other methods, reducing FID by $10$ on ProGAN and at least $5$ on StyleGAN and StyleGAN2. Compared to baselines, it lowers FID by over $15$, with reductions of $25$ compared to \enjoyGAN\ and \MDSE\ and $20$ compared to \AdaTans. \MyMethod\ also excels in attributes like Age, Hairstyle, Pose, and Gender, generating more realistic images and consistently achieving lower LPIPS values, indicating improved quality.}

{In Table~\ref{tab:IS}, \MyMethod\ shows better or comparable IS scores across the three GAN models and performs well on SSIM, surpassing ProGAN and StyleGAN2, and matching \MDSE\ on StyleGAN. It improves SSIM by up to $0.1$ on StyleGAN2, $0.08$ on StyleGAN, and $0.05$ on ProGAN. Additionally, \MyMethod\ outperforms others in editing Age and Gender attributes, maintaining higher structural fidelity. Table~\ref{tab:diffusion} compares \MyMethod\ with diffusion-based methods, showing that it achieves top performance in all metrics while being faster in inference, making it more practical for real-time facial attribute editing. In summary, \MyMethod\ provides high-quality results, outperforming supervised methods despite being unsupervised, and is ideal for real-time and large-scale facial attribute editing.}
\begin{table*}[!htb]
  \centering
  \caption{FID and LPIPS comparisons of different methods on three pre-trained GAN Models. "{--}" indicates that the method does not provide implementation details for editing this facial attribute. Standard deviations are shown in parentheses.}
  \renewcommand{\arraystretch}{1.6}
~\label{tab:FID}
  \resizebox{\textwidth}{!}{
  \begin{tabular}{c|c|cc|cc|cc|cc|cc|cc|cc|cc}
    \toprule
  \hline  
  \multirow{3}{*}{{Model}} & \multirow{3}{*}{{Attribute}} 
  & \multicolumn{6}{c|}{Supervised Method} & \multicolumn{10}{c}{Unsupervised Method} \\ \cline{3-18}
  & \multicolumn{1}{c|}{} & \multicolumn{2}{c|}{\makecell{\InterFaceGAN\ \\ \citep{InterFaceGAN2020}}} & \multicolumn{2}{c|}{\makecell{\enjoyGAN\ \\ \citep{EnjoyGAN2021}}} & \multicolumn{2}{c|}{\makecell{\MDSE\ \\ \citep{MDSE2023}}} & \multicolumn{2}{c|}{\makecell{\GANspace\ \\ \citep{GANSpace2020}}} & \multicolumn{2}{c|}{\makecell{\sefa\ \\ \citep{SEFA2021}}}  & \multicolumn{2}{c|}{\makecell{\AdaTans\ \\ \citep{AdaTrans2023}}} & \multicolumn{2}{c|}{\makecell{\SDFlow\ \\ \citep{SDFlow2024ICASSP}}}  & \multicolumn{2}{c}{{\textbf{\MyMethod\ (Our)}}}  \\    
  
  \cline{3-18}
  & \multirow{6}*{} & FID $\downarrow$ & LPIPS $\downarrow$ & FID $\downarrow$ & LPIPS $\downarrow$ & FID $\downarrow$ & LPIPS $\downarrow$ & FID $\downarrow$ & LPIPS $\downarrow$ & FID $\downarrow$ & LPIPS $\downarrow$ & FID $\downarrow$ & LPIPS $\downarrow$ & FID $\downarrow$ & LPIPS $\downarrow$ & FID $\downarrow$ & LPIPS $\downarrow$ \\ 
  \hline
  \multirow{6}*{\textit{StyleGAN}}
      & Age 
        & 68.93 (0.03) & 0.3333 (0.02) & 74.32 (0.02) & 0.3856 (0.03) & 87.54 (0.01) & 0.4166 (0.02) & 64.86 (0.02) & 0.3119 (0.01) & 60.91 (0.03) & 0.2933 (0.02) & 62.34 (0.02) & 0.2644 (0.03) & \underline{59.39 (0.02)} & \underline{0.2425 (0.01)} & \textbf{45.72 (0.01)} & \textbf{0.2387 (0.02)}\\
        & Gender 
        & 82.38 (0.03) & 0.3814 (0.02) & 76.32 (0.02) & 0.3611 (0.03) & 76.99 (0.01) & 0.3825 (0.02) & 71.99 (0.02) & 0.3741 (0.01) & 70.79 (0.03) & 0.3428 (0.02) & 78.59 (0.02) & 0.3323 (0.03) & \underline{68.08 (0.02)} & \underline{0.2641 (0.01)} & \textbf{55.74 (0.01)} & \textbf{0.2156 (0.02)}\\
        & Hairstyle 
        & 68.63 (0.01) & 0.3154 (0.02) & 75.66 (0.02) & 0.3813 (0.03) & {--} & {--} & \underline{38.05 (0.02)} & \underline{0.1937 (0.01)} & \textbf{25.13 (0.01)} & \textbf{0.1448 (0.02)} & 72.55 (0.03) & 0.3537 (0.02) & 66.5 (0.02) & 0.2536 (0.01) & 43.21 (0.02) & 0.2322 (0.02)\\
        & Pose 
        & 75.65 (0.00) & 0.4112 (0.02) & {--} & {--} & 73.11 (0.01) & 0.3548 (0.03) & \underline{53.95 (0.02)} & \underline{0.2526 (0.01)} & 65.63 (0.03) & 0.3721 (0.02) & {--} & {--} & {--} & {--} & \textbf{41.57 (0.01)} & \textbf{0.2213 (0.02)}\\
        & Smile 
        & \underline{30.65 (0.02)} & \underline{0.1718 (0.01)} & 63.39 (0.03) & 0.3324 (0.02) & 42.36 (0.01) & 0.2133 (0.03) & \textbf{25.87 (0.01)} & \textbf{0.2133 (0.02)} & 59.13 (0.02) & 0.2617 (0.01) & 68.21 (0.03) & 0.3269 (0.02) & 57.42 (0.02) & 0.2235 (0.01) & 39.41 (0.01) & \textbf{0.1231 (0.02)} \\
        & \textbf{Avg.} 
        & 65.25 (0.02) & 0.3226 (0.02) & 72.42 (0.02) & 0.3651 (0.03) & 70.00 (0.01) & 0.3418 (0.02) & \underline{50.94 (0.02)} & \underline{0.2691 (0.01)} & 56.32 (0.03) & 0.2829 (0.02) & 70.42 (0.02) & 0.3193 (0.03) & 62.85 (0.02) & \underline{0.2459 (0.01)} & \textbf{45.13 (0.01)} & \textbf{0.2062 (0.02)}\\
        \midrule
  \multirow{6}*{\textit{StyleGAN2}}
      & Age 
        & 69.39 (0.02) & 0.3724 (0.03) & 65.24 (0.02) & 0.3045 (0.02) & 72.19 (0.03) & 0.3649 (0.02) & 63.96 (0.02) & 0.3356 (0.03) & 56.35 (0.03) & 0.3147 (0.02) & 74.03 (0.02) & 0.2646 (0.03) & \underline{52.41 (0.02)} & \textbf{0.1844 (0.03)} & \textbf{48.97 (0.02)} & \underline{0.2423 (0.03)}\\
        & Gender 
        & 62.77 (0.03) & 0.3329 (0.02) & 62.33 (0.02) & 0.2836 (0.03) & 68.37 (0.03) & 0.3644 (0.02) & 67.37 (0.02) & 0.3438 (0.03) & \underline{51.24 (0.03)} & 0.2722 (0.02) & 79.92 (0.02) & 0.2617 (0.03) & 59.00 (0.02) & \textbf{0.2234 (0.03)} & \textbf{47.75 (0.03)} & \underline{0.2534 (0.02)}\\
        & Hairstyle 
        & 61.16 (0.03) & 0.3418 (0.02) & 61.28 (0.02) & 0.3219 (0.03) & {--} & {--} & 61.22 (0.03) & 0.3428 (0.02) & 49.95 (0.02) & 0.2846 (0.03) & 81.85 (0.03) & 0.2518 (0.02) & \textbf{45.57 (0.02)} & \underline{0.2836 (0.03)} & \underline{49.07 (0.03)} & \textbf{0.2411 (0.02)}\\
        & Pose 
        & 74.95 (0.02) & 0.3764 (0.03) & {--} & {--} & 63.92 (0.03) & 0.4477 (0.02) & 66.21 (0.02) & 0.3519 (0.03) & \underline{60.50 (0.03)} & \underline{0.3112 (0.02)} & {--} & {--} & {--} & {--} & \textbf{46.82 (0.02)} & \textbf{0.2621 (0.03)} \\
        & Smile 
        & 69.95 (0.03) & 0.3822 (0.02) & 62.79 (0.02) & 0.2936 (0.03) & 66.03 (0.03) & 0.2628 (0.02) & 73.21 (0.02) & 0.3825 (0.03) & 59.29 (0.03) & 0.3123 (0.02) & \underline{44.18 (0.02)} & \underline{0.2534 (0.03)} & 50.21 (0.03) & 0.2622 (0.02) & \textbf{43.54 (0.02)} & \textbf{0.1814 (0.03)}\\
        & \textbf{Avg.} 
        & 67.64 (0.03) & 0.3611 (0.02) & 64.17 (0.02) & 0.3009 (0.03) & 62.17 (0.03) & 0.3599 (0.02) & 65.39 (0.02) & 0.3513 (0.03) & \underline{55.47 (0.03)} & 0.2990 (0.02) & 75.46 (0.02) & 0.2579 (0.03) & 51.80 (0.02) & \underline{0.2384 (0.03)} & \textbf{47.23 (0.03)} & \textbf{0.2361 (0.02)}\\
        \midrule

  \multirow{6}*{\textit{ProGAN}}
      & Age 
      & 79.10 (0.01) & 0.3224 (0.00) & 78.32 (0.00) & 0.3212 (0.01) & 77.15 (0.00) & 0.3487 (0.02) & \underline{61.46 (0.00)} & 0.3021 (0.03) & 72.91 (0.02) & 0.3031 (0.00) & 70.65 (0.05) & 0.2938 (0.02) & 62.95 (0.03) &\underline{0.2646 (0.00)} & \textbf{60.27 (0.01)} & \textbf{0.2617 (0.00)}\\
      & Gender 
      & 89.52 (0.02) & 0.3645 (0.01) & 77.21 (0.01) & 0.3254 (0.00) & 77.52 (0.02) & 0.4132 (0.03) & \underline{58.15 (0.05)} & \underline{0.2841 (0.00)} & 93.97 (0.00)& 0.3546 (0.02) & 76.47 (0.03) & 0.3026 (0.00)  & 74.35 (0.00) & 0.3149 (0.01) & \textbf{56.81 (0.01)} & \textbf{0.2622 (0.00)}\\
      & Hairstyle 
      & \underline{50.31 (0.01)} & 0.2511 (0.00) & 69.64 (0.00) & 0.3124 (0.00) & {--} & {--} & 72.67 (0.01)  & 0.3128 (0.02) & 56.23 (0.01) & 0.2324 (0.00) & 76.77 (0.01) & 0.3011 (0.02) & 60.76 (0.01) & \underline{0.2426 (0.03)} & \textbf{48.17 (0.01)} & \textbf{0.2311 (0.00)} \\
      & Pose 
      & 65.06 (0.00) & 0.3933 (0.03) & {--} & {--} & 66.93 (0.05) & 0.5444 (0.00) & \underline{61.96 (0.03)} & \underline{0.2697 (0.01)}& 67.33 (0.02) & 0.3824 (0.00) & {--}  & {--} & {--}  & {--} & \textbf{52.22 (0.03)} & \textbf{0.2431 (0.00)}\\
      & Smile 
      & \underline{45.82 (0.00)} & \underline{0.2118 (0.02)} & 63.21 (0.03) & 0.2936 (0.02) & 74.31 (0.05) & 0.4097 (0.02) & 60.88 (0.01) & 0.2756 (0.00) & 47.22 (0.00) & \textbf{0.1945 (0.03)}  & 75.86 (0.01) & 0.3226 (0.00) & 60.66 (0.02) & 0.2329 (0.01) & \textbf{44.43 (0.02)} & 0.2236 (0.01)\\
      & \textbf{Avg.} 
      & 65.96 (0.01) & 0.3086 (0.01) & 72.10 (0.01) & 0.3132 (0.00) & {73.98 (0.03)} & 0.4290 (0.02) & \underline{62.76 (0.02)} & 0.2889 (0.01) & 67.53 (0.01) & 0.2934 (0.00) & 74.94 (0.02) & 0.3050 (0.01) & 64.68 (0.02) & \underline{0.2638 (0.00)} & \textbf{52.65 (0.01)} & \textbf{0.2443 (0.00)}\\
      \hline
      \bottomrule
  \end{tabular}
  }
\end{table*}
\begin{table*}[!htb]
  \centering
  \caption{IS and SSIM comparisons of different methods on three pre-trained GAN Models. "{--}" indicates that the method does not provide implementation details for editing this facial attribute. Standard deviations are shown in parentheses.}
  \renewcommand{\arraystretch}{1.6}
~\label{tab:IS}
  \resizebox{\textwidth}{!}{
  \begin{tabular}{c|c|cc|cc|cc|cc|cc|cc|cc|cc} 
    \toprule
  \hline  
  \multirow{3}{*}{{Model}} & \multirow{3}{*}{{Attribute}} 
  & \multicolumn{6}{c|}{Supervised Method} & \multicolumn{10}{c}{Unsupervised Method} \\ \cline{3-18}
  & \multicolumn{1}{c|}{} & \multicolumn{2}{c|}{\makecell{\InterFaceGAN\ \\ \citep{InterFaceGAN2020}}} & \multicolumn{2}{c|}{\makecell{\enjoyGAN\ \\\citep{EnjoyGAN2021}}} & \multicolumn{2}{c|}{\makecell{\MDSE\ \\\citep{MDSE2023}}} & \multicolumn{2}{c|}{\makecell{\GANspace\ \\\citep{GANSpace2020}}} & \multicolumn{2}{c|}{\makecell{\sefa\ \\\citep{SEFA2021}}}  & \multicolumn{2}{c|}{\makecell{\AdaTans\ \\\citep{AdaTrans2023}}} & \multicolumn{2}{c|}{\makecell{\SDFlow\ \\\citep{SDFlow2024ICASSP}}}  & \multicolumn{2}{c}{{\textbf{\MyMethod\ (Our)}}}  \\    
  
  \cline{3-18}
  &\multirow{6}*{} & IS $\uparrow$ & SSIM $\uparrow$ & IS $\uparrow$& SSIM $\uparrow$& IS $\uparrow$& SSIM $\uparrow$& IS $\uparrow$& SSIM $\uparrow$& IS $\uparrow$& SSIM $\uparrow$& IS $\uparrow$& SSIM $\uparrow$& IS $\uparrow$& SSIM $\uparrow$& IS $\uparrow$& SSIM $\uparrow$\\
  \hline
  \multirow{6}{*}{\textit{StyleGAN}}
      & Age 
          & 2.942 (0.01) & 0.6834 (0.01) & 2.975 (0.01) & 0.7148 (0.02) & 3.014 (0.02) & 0.7212 (0.01) & \underline{3.081 (0.01)} & 0.7036 (0.01) & 2.894 (0.02) & 0.7348 (0.01) & 2.421 (0.01) & \underline{0.7434 (0.01)} & 2.458 (0.02) & 0.7221 (0.01) & \textbf{3.391 (0.01)} & \textbf{0.7726 (0.02)}\\
          & Gender 
          & 3.114 (0.01) & 0.6331 (0.01) & 3.113 (0.02) & 0.7240 (0.01) & 3.242 (0.01) & \textbf{0.7644 (0.02)} & 2.854 (0.01) & 0.6549 (0.01) & \underline{3.275 (0.01)} & 0.6789 (0.02) & 2.410 (0.01) & 0.6824 (0.01) & 2.534 (0.02) & 0.6930 (0.01) & \textbf{3.301 (0.01)} & \underline{0.7344 (0.02)}\\
          & Hairstyle 
          & \underline{3.062 (0.01)} & 0.7146 (0.01) & 2.792 (0.02) & 0.7123 (0.01) & {--} & {--} & 3.033 (0.01) & \underline{0.8235 (0.01)} & \textbf{3.101 (0.01)} & 0.7847 (0.02) & 2.553 (0.01) & 0.6627 (0.01) & 2.454 (0.02) & 0.7069 (0.01) & 2.362 (0.01) & \textbf{0.8324 (0.02)}\\
          & Pose 
          & 3.145 (0.01) & 0.6131 (0.01) & {--} & {--} & 2.943 (0.02) & \underline{0.7879 (0.01)} & 3.091 (0.01) & 0.7569 (0.01) & \underline{3.230 (0.01)} & 0.6445 (0.02) & {--}  & {--} & {--} & {--} & \textbf{3.263 (0.01)} & \textbf{0.7986 (0.01)}\\
          & Smile 
          & 3.143 (0.01) & 0.8429 (0.01) & 2.810 (0.02) & 0.7422 (0.01) & 3.064 (0.01) & \underline{0.8848 (0.02)} & 2.991 (0.01) & 0.8747 (0.01) & \underline{3.153 (0.01)} & 0.7545 (0.01) & 2.664 (0.01) & 0.6832 (0.01) & 2.571 (0.02) & 0.7337 (0.01) & \textbf{3.412 (0.01)} & \textbf{0.8913 (0.02)}\\
          & \textbf{Avg.} 
          & 3.081 (0.01) & 0.6974 (0.01) & 2.923 (0.02) & 0.7233 (0.01) & 3.066 (0.01) & \underline{0.7896 (0.02)} & {3.010 (0.01)} & 0.7627 (0.01) & \underline{3.131 (0.01)} & 0.7195 (0.01) & 2.512 (0.01) & 0.6929 (0.01) & 2.504 (0.02) & 0.7139 (0.01) & \textbf{3.146 (0.01)} & \textbf{0.8059 (0.02)}\\ 
          \midrule
  \multirow{6}{*}{\textit{StyleGAN2}}
       & Age 
        & \underline{2.794 (0.01)} & 0.7625 (0.02) & 2.620 (0.00) & 0.7927 (0.01) & 2.411 (0.02) & 0.7688 (0.00) & 2.683 (0.01) & 0.7825 (0.02) & 2.673 (0.01) & 0.8129 (0.00) & 2.734 (0.02) & 0.7936 (0.01) & 1.843 (0.00) & \underline{0.8122 (0.01)} & \textbf{2.934 (0.01)} & \textbf{0.8428 (0.02)}\\
        & Gender 
        & 2.501 (0.02) & 0.7897 (0.01) & \underline{2.951 (0.01)} & 0.8128 (0.02) & 2.453 (0.00) & 0.7432 (0.01) & 2.504 (0.02) & 0.7834 (0.00) & 2.501 (0.01) & \underline{0.8377 (0.02)} & 2.740 (0.02) & 0.7943 (0.01) & 1.685 (0.00) & 0.7733 (0.02) & \textbf{3.010 (0.01)} & \textbf{0.8537 (0.02)}\\
        & Hairstyle 
        & 2.553 (0.01) & 0.7824 (0.02) & 2.534 (0.00) & 0.7856 (0.01) & {--} & {--} & \underline{2.874 (0.01)} & 0.7866 (0.02) & 2.564 (0.02) & \underline{0.8326 (0.01)} & {2.692 (0.01)} & 0.7844 (0.00) & 1.863 (0.02) & 0.7824 (0.01) & \textbf{2.883 (0.01)} & \textbf{0.8411 (0.02)}\\
        & Pose 
        & {2.814 (0.02)} & 0.7631 (0.01) & {--} & {--} & 2.683 (0.01) & 0.6839 (0.02) & 2.631 (0.00) & 0.7749 (0.01) & \underline{2.859 (0.01)} & \underline{0.8128 (0.02)} & {--}  & {--} & {--} & {--} & \textbf{2.991 (0.01)} & \textbf{0.8425 (0.02)}\\
        & Smile 
        & 2.649 (0.01) & 0.7628 (0.02) & 2.613 (0.00) & 0.7844 (0.01) & 2.543 (0.02) & \underline{0.8446 (0.01)} & \underline{2.956 (0.01)} & 0.7536 (0.02) & 2.674 (0.01) & 0.8023 (0.00) & 2.846 (0.02) & 0.8024 (0.01) & 1.853 (0.00) & 0.7177 (0.02) & \textbf{3.111 (0.01)} & \textbf{0.8544 (0.02)}\\
        & \textbf{Avg.} 
        & 2.662 (0.01) & 0.7721 (0.02) & 2.680 (0.00) & 0.7939 (0.01) & 2.523 (0.02) & 0.7601 (0.00) & {2.730 (0.01)} & 0.7762 (0.02) & 2.654 (0.01) & \underline{0.8197 (0.02)} & \underline{2.753 (0.01)} & 0.7937 (0.00) & 1.811 (0.02) & 0.7714 (0.01) & \textbf{2.986 (0.01)} & \textbf{0.8469 (0.02)}\\
        \midrule

   \multirow{6}{*}{\textit{ProGAN}}
    & Age 
    & 2.529 (0.01) & 0.6944 (0.01) & 2.012 (0.01) & 0.6848 (0.01) & \underline{2.564 (0.01)} & \underline{0.7547 (0.01)} & 2.101 (0.01) & 0.7044 (0.01) & {2.563 (0.01)} & 0.7135 (0.01) & 2.087 (0.01) & 0.7233 (0.01) & 2.061 (0.01) & 0.7425 (0.01) & \textbf{2.613 (0.01)} & \textbf{0.7645 (0.01)}\\
    & Gender 
    & \textbf{2.711 (0.01)} & 0.6548 (0.01) & 2.483 (0.01) & \textbf{0.7489 (0.01)} & 2.310 (0.01) & 0.7149 (0.01) & 2.213 (0.01) & 0.7134 (0.01) & {2.712 (0.01)} & 0.6587 (0.01) & 2.323 (0.01) & 0.7133 (0.01) & 2.154 (0.01) & 0.6924 (0.01) & \underline{2.633 (0.01)} & \underline{0.7426 (0.01)}\\
    & Hairstyle 
    & 2.320 (0.01) & 0.7677 (0.01) & 2.331 (0.01) & 0.7624 (0.01) & {--} & {--} & \underline{2.725 (0.01)} & 0.6946 (0.01) & 2.303 (0.01) & \underline{0.7811 (0.01)} & 2.270 (0.01) & 0.7133 (0.01) & 2.033 (0.01) & 0.7740 (0.01) & \textbf{2.832 (0.01)} & \textbf{0.7824 (0.01)}\\
    & Pose 
    & 2.401 (0.01) & 0.6137 (0.01) & {--} & {--} & \underline{2.510 (0.01)} & \textbf{0.7722 (0.01)} & 2.493 (0.01) & 0.7346 (0.01) & 2.479 (0.01) & 0.6329 (0.01) & {--} & {--} & {--} & {--} & \textbf{2.513 (0.01)} & \underline{0.7629 (0.01)}\\
    & Smile 
    & 2.577 (0.01) & 0.8028 (0.01) & \underline{2.663 (0.01)} & \underline{0.8129 (0.01)} & 2.513 (0.01) & 0.7749 (0.01) & {2.623 (0.01)} & 0.7241 (0.01) & 2.614 (0.01) & \textbf{0.8212 (0.01)} & 2.184 (0.01) & 0.6946 (0.01) & 1.963 (0.01) & 0.7846 (0.01) & \textbf{2.683 (0.01)} & 0.7727 (0.01)\\
    & \textbf{Avg.} 
    & 2.508 (0.01) & 0.7067 (0.01) & 2.372 (0.01) & 0.7523 (0.01) & {2.474 (0.01)} & \underline{0.7542 (0.01)} & 2.431 (0.01) & 0.7142 (0.01) & \underline{2.534 (0.01)} & 0.7215 (0.01) & 2.216 (0.01) & 0.7111 (0.01) & 2.053 (0.01) & 0.7484 (0.01) & \textbf{2.655 (0.01)} & \textbf{0.7650 (0.01)}\\
    \hline
    \bottomrule
    \end{tabular}
  }
\end{table*}

\begin{table*}[!htb]
\centering
\caption{{A comparison between our proposed \MyMethod\ and state-of-the-art diffusion-based facial attribute editing methods in terms of SSIM, FID, LPIPS, and inference time, where \MyMethod\ is tested on the StyleGAN backbone. Standard deviations are shown in parentheses.}}
\setlength{\tabcolsep}{4pt}
\renewcommand{\arraystretch}{1.4}
\resizebox{\textwidth}{!}{
\begin{tabular}{cc|ccc|ccc|ccc|ccc}
    \toprule
    \hline
    \multirow{3}{*}{{{}}} & \multirow{3}{*}{{{Attribute}}} 
    & \multicolumn{9}{c|}{{Diffusion-based Method}} & \multicolumn{3}{c}{{GAN-based Method (StyleGAN Backbone)}} \\ \cline{3-14}
    & \multicolumn{1}{c|}{} & \multicolumn{3}{c|}{{\makecell{DDS\ \citep{dds2023}}}} & \multicolumn{3}{c|}{{\makecell{CDS\ \citep{cds2024}}}} & \multicolumn{3}{c|}{{\makecell{FPE\ \citep{FPE2024}}}} & \multicolumn{3}{c}{{\makecell{\textbf{\MyMethod\ (Our)}}}}  \\   
    \cline{3-14}
    &\multirow{4}*{{}} & {SSIM $\uparrow$} & {FID $\downarrow$} & {LPIPS $\downarrow$}& {SSIM $\uparrow$} & {FID $\downarrow$} & {LPIPS $\downarrow$} & {SSIM $\uparrow$} & {FID $\downarrow$} & {LPIPS $\downarrow$} & {SSIM $\uparrow$} & {FID $\downarrow$} & {LPIPS $\downarrow$}\\ \midrule
    & {Age} & {0.7931 (0.03)} & {86.34 (0.05)} & {\textbf{0.1721 (0.01)}} & {0.7121 (0.00)} & {70.28 (0.01)} & {0.2802 (0.02)} &  {0.6823 (0.02)} & {\underline{64.36 (0.00)}} & {0.2726 (0.02)} & {\textbf{0.7726 (0.02)}} & {\textbf{45.72 (0.01)}} & {\underline{0.2387 (0.02)}} \\
    & {Gender} & {0.4231 (0.01)} & {83.57 (0.03)} & {0.4122 (0.00)} & {\underline{0.6946 (0.02)}} & {\underline{69.17 (0.03)}} & {\underline{0.3024 (0.03)}} & {0.3944 (0.05)} & {72.43 (0.00)} & {0.4936 (0.03)} & {\textbf{0.7344 (0.02)}} & {\textbf{55.74 (0.01)}} & {\textbf{0.2156 (0.02)}} \\ 
    & {Smile} & {0.5234 (0.01)} & {83.36 (0.00)} & {0.2853 (0.02)} & {0.7647 (0.00)} & {80.48 (0.01)} & {0.2159 (0.01)} & {0.5936 (0.00)} & {\underline{73.98 (0.02)}} & {0.2710 (0.00)} & {\underline{0.8913 (0.02)}} & {\textbf{39.41 (0.01)}} & {\textbf{0.1231 (0.02)}} \\ 
    & {\textbf{Avg.}} & {0.5799 (0.01)} & {84.42 (0.02)} & {0.2899 (0.01)} & {\underline{0.7238 (0.00)}} & {73.31 (0.01)} & {0.2662 (0.02)} & {{0.5568 (0.01)}} & {\underline{70.26 (0.00)}} & {0.3457 (0.01)} & {\textbf{0.7994 (0.02)}} & {\textbf{46.96 (0.01)}}  & {\textbf{0.1925 (0.02)}} \\
    \bottomrule
    & \multicolumn{13}{c}{{Inference Time $\downarrow$}} \\ \midrule
    & {\textbf{Avg.}} & \multicolumn{3}{c|}{{17.43s}} & \multicolumn{3}{c|}{{16.31s}} & \multicolumn{3}{c|}{{7.42s}} & \multicolumn{3}{c}{{\textbf{1.42s}}} \\ \hline
    \bottomrule
\end{tabular}
}
\label{tab:diffusion}
\end{table*}

\section{Discussions}\label{sec51}
{In this section, we discuss (i) a comprehensive discussion on the effectiveness of each component in~\MyMethod\ (ablation study), (ii) the sensitivity of the regularization parameter in the proposed \MyAlgorithm, (iii) user study and practical applications, as well as the limitations of \MyMethod.}

\subsection{Ablation Study and Discussion on Sensitivity of Regularization Parameters for \MyAlgorithm}
{This section specifies two complementary protocols: an ablation study to disentangle the effect of each component, and a sensitivity analysis to assess the robustness of the full model to $(\alpha,\lambda)$. All settings follow the unified evaluation in Section~\ref{sec:exp_details}. Together, these two analyses reveal both the functional necessity of each term and the stability of the optimization under varying regularization strengths.} 
\subsubsection{Experimental Design}
{For the ablation study, we remove one component at a time under the unified evaluation protocol and compare against the full \MyMethod, reporting Avg corr (inter-attribute correlation from Eq.~\eqref{corrcompute}; lower is better) together with FID, LPIPS, and SSIM (Table~\ref{tab:ablation}). To examine the specific contribution of each term, we construct three ablation settings as follows:} 
(i) \emph{w/o boundary} sets $\lambda{=}0$ in Eq.~\eqref{objFun1}/\eqref{objFun3}; 
(ii) \emph{w/o non-negativity} removes the $\B{F}$ projection (set $\alpha=0$); 
(iii) \emph{w/o orthogonality} drops $\B{W}^\top\B{W}{=}\B{I}$ and updates $\B{W}\!\leftarrow\!\B{A}{=}\B{Z}\B{P}^\top{+}\alpha\B{F}{-}\lambda\B{S}$ (with column $\ell_2$ normalization for stability). \par

For the sensitivity analysis of the regularization parameters $\alpha$ and $\lambda$ in \MyAlgorithm, we vary one parameter while keeping the other fixed within a small grid and measure the average changes in four metrics: FID, LPIPS, IS, and SSIM. 
{Concretely, for the $\lambda$-sweep, we set $\alpha\in\{0.1,\,0.2,\,0.5,\,1,\,2,\,3,\,5\}$ and vary $\lambda\in\{1,\,2,\,4,\,5\}$; for the $\alpha$-sweep we fix $\lambda$ on the same grid and vary $\alpha$ on the above set.}
For each $(\alpha,\lambda)$ pair (e.g., $(\alpha{=}0.1,\lambda{=}1)$), semantic vectors are computed per attribute and {$100$} edited images are generated; we then average FID, LPIPS, IS, and SSIM over attributes. 
{These sensitivity results characterize how stable the model remains once all regularization components are included, and thus complement the ablation findings in Table~\ref{tab:ablation} (see Figs.~\ref{fig:sens1}-\ref{fig:sens2}).}

\subsubsection{{Ablation Study}}
{
The ablation results are summarized in Table~\ref{tab:ablation}. 
From these results, it is evident that each regularization component plays a distinct and complementary role in learning the semantic vectors. 
Specifically,}
\begin{itemize}
  \item {\textbf{w/o boundary:} Removing the boundary constraint weakens the separation between semantic attributes, leading to higher inter-attribute correlation (larger Avg corr) and degraded FID and LPIPS. This demonstrates that the boundary term is crucial for maintaining attribute disentanglement and perceptual fidelity.}
  \item {\textbf{w/o non-negativity:} Without the non-negativity constraint, semantic vectors deviate from realistic feature directions. Although overall visual quality remains reasonable, higher Avg corr and LPIPS indicate stronger attribute coupling and reduced perceptual consistency.}
  \item {\textbf{w/o orthogonality:} Dropping the orthogonality constraint slightly increases interference among attributes, as reflected by lower SSIM and moderately higher correlation. This confirms that orthogonality helps preserve structural consistency across generated images.}
\end{itemize}\par
{
Overall, these findings verify that each regularization term contributes uniquely to stabilizing training and enhancing disentanglement. 
Removing any component increases inter-attribute correlation and degrades image quality, while the subsequent sensitivity analysis further demonstrates that \MyAlgorithm\ remains robust under moderate variations of $(\alpha,\lambda)$. 
Therefore, \MyMethod\ can be reliably applied to facial-attribute editing with minimal parameter tuning, offering both flexibility and robustness in practical applications.}
\begin{table}[h]
\centering
\caption{{Ablation study on U\text{-}Face components (averaged over selected attributes). Lower corr/FID/LPIPS is better, higher SSIM is better.}}
\label{tab:ablation}
\small
\begin{tabular}{lcccc}
\toprule
\hline
\textbf{{Variant}} & \textbf{{Avg corr} $\downarrow$} & \textbf{{FID} $\downarrow$} & \textbf{{LPIPS} $\downarrow$} & \textbf{{SSIM} $\uparrow$} \\
\midrule
{w/o boundary ($\lambda{=}0$)} & {0.274} & {52.34} & {0.2234} & {0.7063} \\
{w/o non-negativity ($\alpha{=}0$)} & {0.155} & {50.34} & {0.2466} & {0.7938} \\
{w/o orthogonality} & {0.221} & {51.22} & {0.2214} & {0.7541} \\
\textbf{{U\text{-}Face (full)}} & \textbf{{0.096}} & \textbf{{45.13}} & \textbf{{0.2062}} & \textbf{{0.8059}} \\
\hline
\bottomrule
\end{tabular}
\end{table}
\subsubsection{Sensitivity Analysis}
The experimental results in Fig.~\ref{fig:sens1} indicate that the proposed \MyAlgorithm\ is not sensitive to the parameters $\alpha$ and $\lambda$ in the FID, IS, SSIM, and LPIPS metrics. Specifically, when $\lambda=1$, the IS, FID, LPIPS, and SSIM values vary within a small range as $\alpha$ increases. Similarly, when $\lambda$ is set to $2$, $4$, or $5$, the variations in these metrics remain relatively small. This suggests that, for a fixed $\lambda$, different values of $\alpha$ have a limited impact on the quality of the original facial images.\par
As shown in Fig.~\ref{fig:sens2}, when the regularization parameter $\alpha$ is fixed, the regularization parameter $\lambda$ exhibits similar sensitivity. Specifically, the IS metric remains approximately $0.1$, the maximum change in the FID metric remains below $2.5$, and the impact of $\lambda$ on the LPIPS and SSIM metrics is even smaller, with variations not exceeding $0.04$ and $0.05$, respectively.\par
\begin{figure*}[h]
  \centering
  \includegraphics[width=\linewidth]{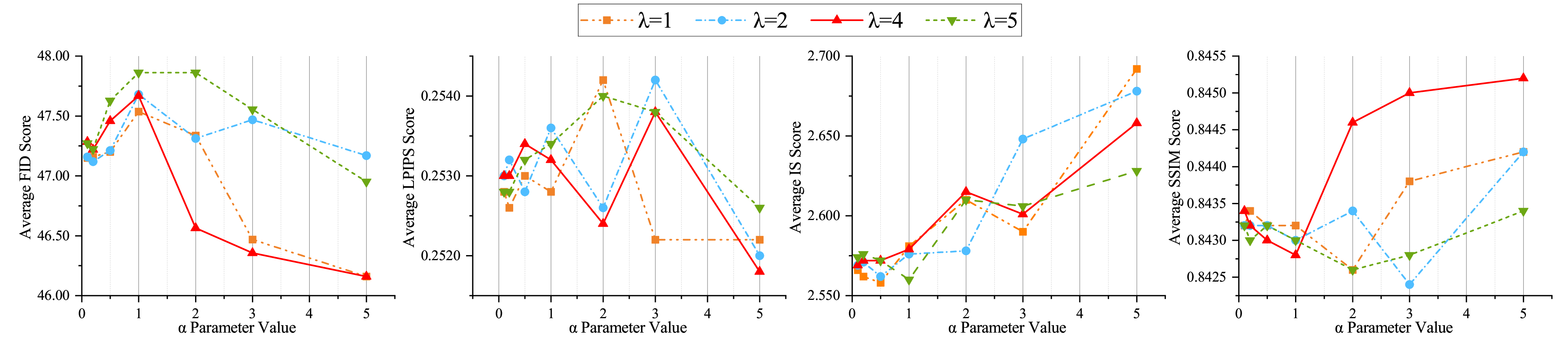}

  \caption{ Comparison of the sensitivity for the regularization parameter $\lambda$ in \MyAlgorithm\ using the average values of FID, LPIPS, IS, and SSIM~. For $\alpha$ values of $\{0.1, 0.2, 0.5, 1, 2, 3, 5\}$, $\lambda$ is varied over the range $\{1, 2, 4, 5\}$.}
~\label{fig:sens1}
\end{figure*}
\begin{figure*}[h]
\centering
\includegraphics[width=\linewidth]{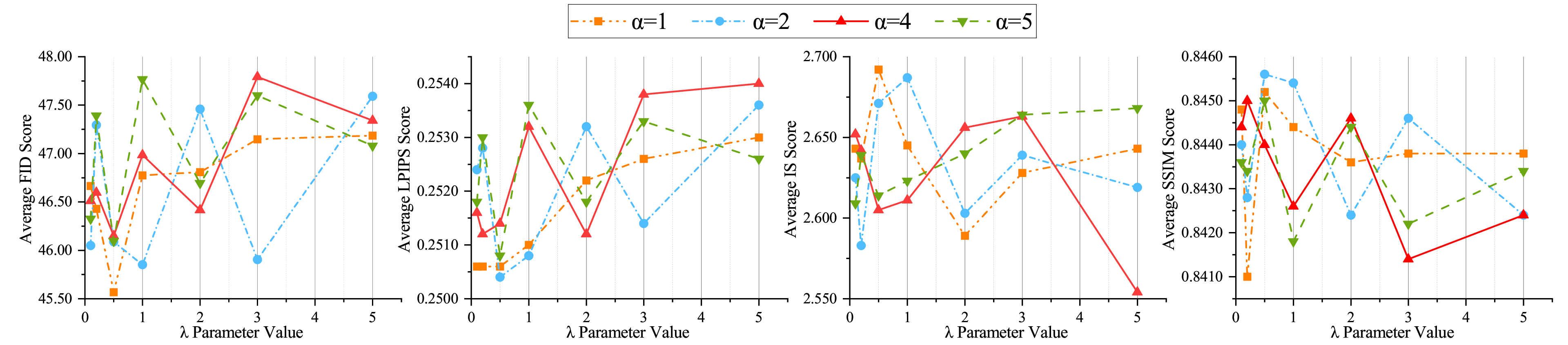}
\caption{Comparison of the sensitivity for the regularization parameter $\alpha$ in \MyAlgorithm\ using the average values of FID, LPIPS, IS, and SSIM~. For $\lambda$ values of $\{0.1, 0.2, 0.5, 1, 2, 3, 5\}$, $\alpha$ is varied over the range $\{1, 2, 4, 5\}$.}
~\label{fig:sens2}
\end{figure*}
In summary, the above analysis demonstrates that the proposed \MyAlgorithm\ algorithm is robust to variations in regularization parameter values. This means that when applying the \MyMethod\ framework to edit facial attributes, extensive adjustments to these parameters are not required. This not only simpliies the optimization process but also enhances the flexibility andreliability of framework in practical applications.
\begin{table}[h]
  \centering
  \caption{{Application summary. U-Face exposes attribute "sliders” ($\beta$) for controllable edits while monitoring identity drift.}}
  \label{tab:apps}
  \small
  \begin{tabular}{p{0.19\textwidth} p{0.24\textwidth} p{0.25\textwidth} p{0.26\textwidth}}
    \toprule
    \hline
    \textbf{{Scenario}} & \textbf{{I/O}} & \textbf{{User controls}} & \textbf{{Practical notes}} \\
    \midrule
    {Interactive entertainment (game/NPC authoring)} &
    {Input: designer seed or style assets; Output: edited portraits/textures} &
    {Sliders for Age, Gender, Hairstyle, Smile; $\beta\in[-0.3,0.3]$} &
    {Offline authoring; disentangled edits avoid side effects; export to DCC/engine pipelines.} \\
    {Digital avatars (meeting/social apps)} &
    {Input: selfie or short video; Output: avatar variants} &
    {Interactive sliders with preview; presets} &
    {Monitor identity similarity; cap edits if drift is large; latency per Table~\ref{tab:diffusion}.} \\
    {Virtual makeup / try-on} &
    {Input: photo; Output: before/after grid or slider preview} &
    {Smile intensity, Hairstyle volume, Age softening ($\beta$ maps to intensity)} &
    {Reduced entanglement keeps skin/texture consistent; provide reset/original toggle.} \\
    \hline
    \bottomrule
  \end{tabular}
\end{table}
\subsection{Practical Applications}
{Beyond subjective ratings, we summarize how U-Face can be used in three representative scenarios. Table~\ref{tab:apps}  outlines I/O, user controls, and practical notes, linking the study's findings (disentanglement and realism) to deployable workflows.}

{Integration sketch. A typical pipeline uses face detection \& alignment $\rightarrow$ (optional) inversion to the latent space $\rightarrow$ attribute editing via $\beta\,\mathbf{w}_i$ $\rightarrow$ rendering/export. The same controls used in the study directly transfer to these applications.}

\subsection{Limitations}
\label{sec:limitations}

{\textbf{Backbone dependence and dataset shift; transfer to unseen data.} 
\MyMethod\ operates in the latent space of a pre\text{-}trained generator, so its behavior inherits the training distribution of the chosen backbone (e.g., FFHQ vs.\ CelebA\text{-}HQ). Distribution shifts across datasets/backbones and deployment on arbitrary real images—typically requiring an external inversion step—can introduce reconstruction residuals and demographic mismatch, which may reduce fidelity and disentanglement.}

{\textbf{Extreme magnitudes and attribute composition.} Very large $|\beta|$ or sequential/multi\text{-}attribute edits can re\text{-}entangle attributes and induce identity drift or artifacts (see Fig.~\ref{fig:hair_plus}-\ref{fig:smile_plus}).}

{\textbf{Mitigations and future directions.} 
In practice, we cap $|\beta|$, prefer multi\text{-}step small edits, and monitor identity similarity; we also re\text{-}normalize $\B{S}$ and sweep $\lambda$ within a small range to improve stability. While these measures alleviate failures, they do not fully eliminate them under severe shifts or extreme edits. Future work will explore boundary refinement and lightweight domain adaptation of $\B{S}$, as well as tighter integration with inversion for real\text{-}image use.}

\begin{figure}[h]
  \centering
  \begin{subfigure}{0.45\textwidth}
      \centering
      \includegraphics[width=\linewidth]{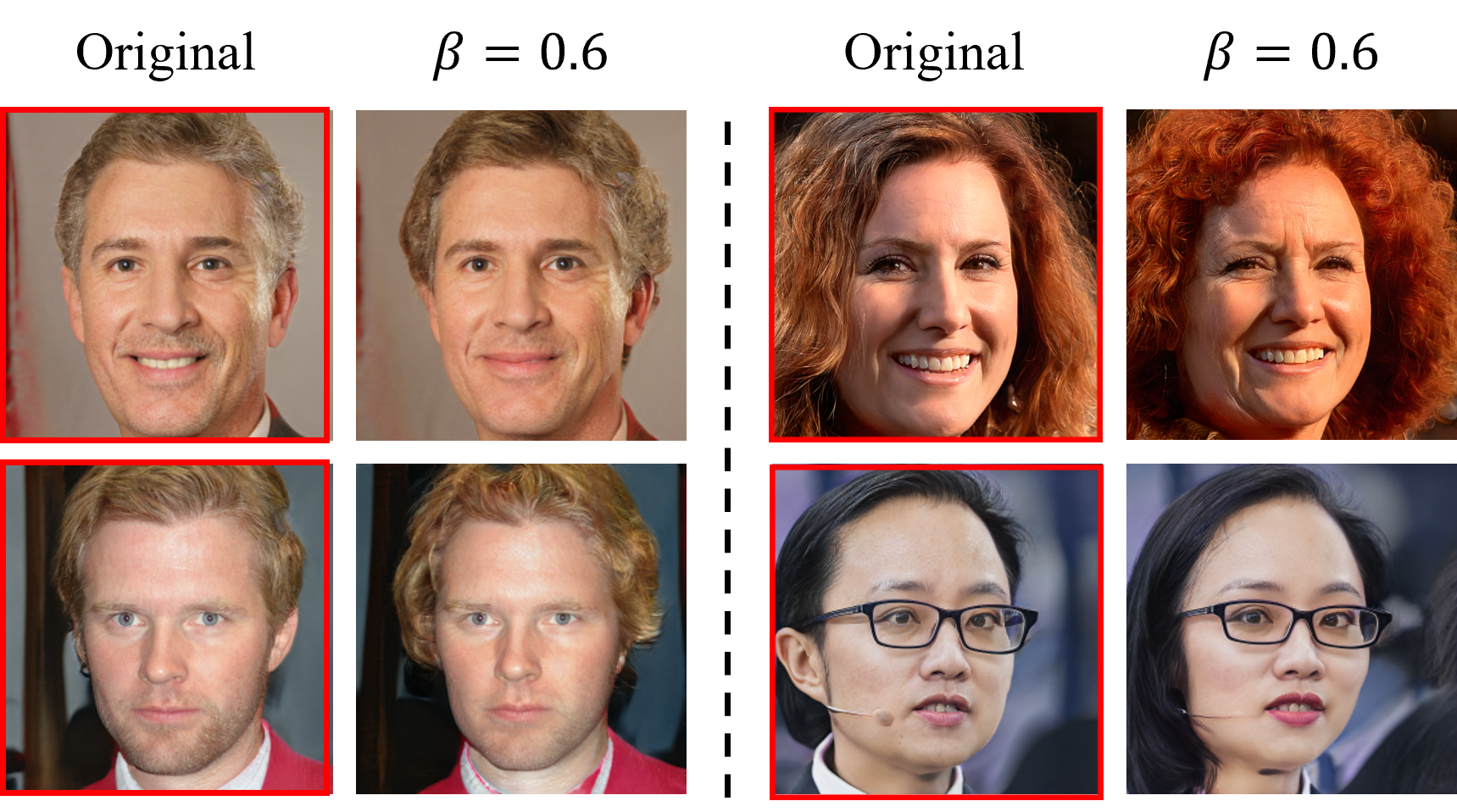}
      \caption{Hairstyle attribute editing visualization}
      \label{fig:hair_plus}
  \end{subfigure}
  \hfill
  \begin{subfigure}{0.45\textwidth}
      \centering
      \includegraphics[width=\linewidth]{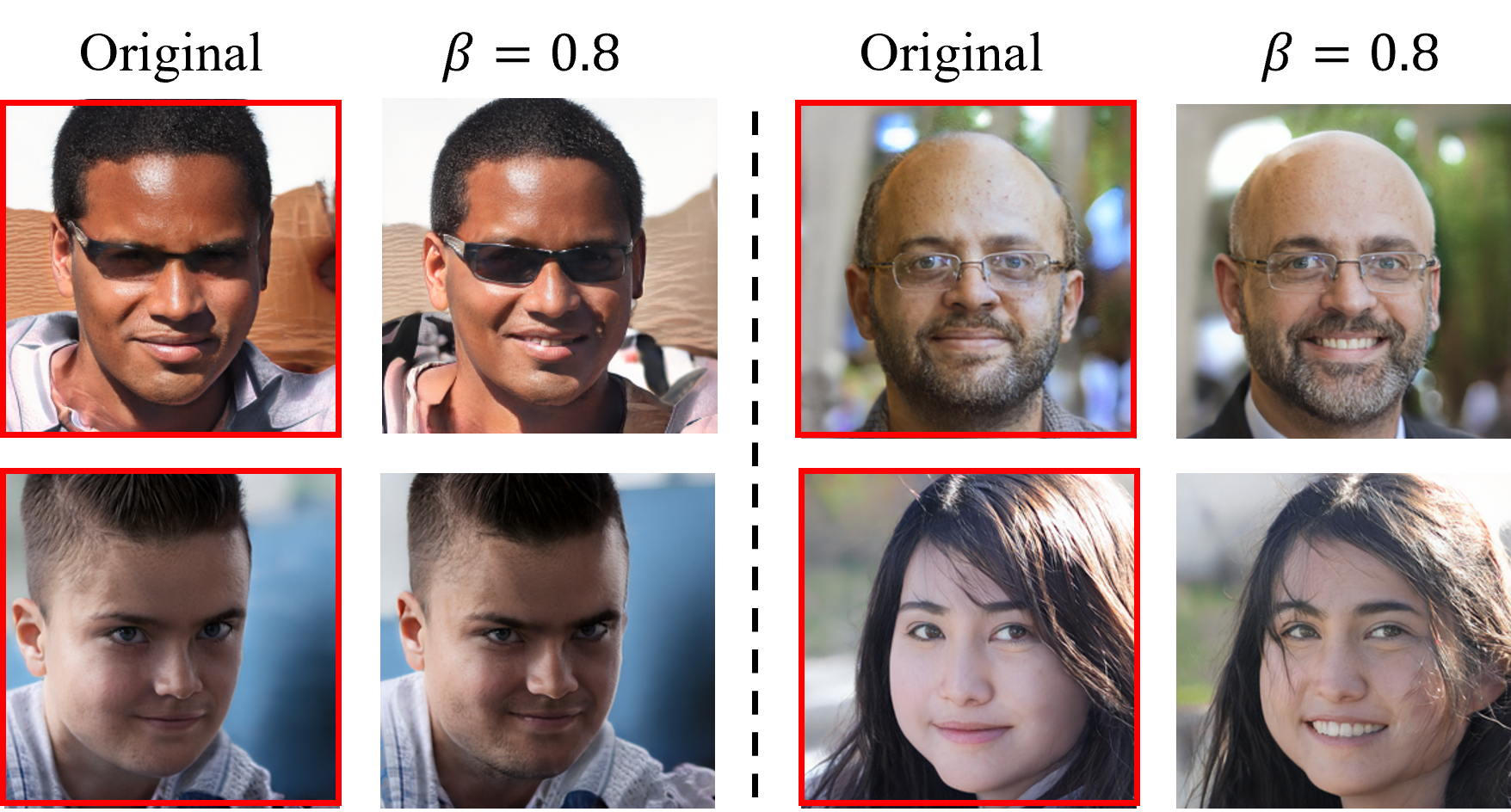}
      \caption{Smile attribute editing visualization}
      \label{fig:smile_plus}
  \end{subfigure}
  \caption{Visualization of Limitations on StyleGAN by \MyMethod.}
  \label{fig:attribute_editing}
\end{figure}

\section{Conclusion}\label{sec:conclusion}
This paper presents \MyMethod, a novel {unsupervised} framework for facial attribute editing that achieves high semantic disentanglement without relying on labeled data. The framework formulates semantic vector learning as a subspace learning problem on the latent vectors of a pre-trained GAN, providing a principled and efficient way to capture meaningful semantic directions. To enhance disentanglement and controllability, we introduced orthogonal {non-negative} constraints to enforce independence among semantic directions and incorporated attribute boundary vectors to steer them away from entangled directions. Furthermore, we developed an alternating iterative algorithm, \MyAlgorithm, with closed-form updates and provable convergence, ensuring both efficiency and stability. Extensive experiments across multiple pre-trained GAN generators demonstrated that \MyMethod\ consistently outperforms state-of-the-art supervised and unsupervised baselines, and achieves editing quality comparable to or better than recent diffusion-based methods while being faster in inference. User studies further corroborated these findings, showing clear advantages in perceptual realism and disentanglement, thereby validating the robustness of the framework.\par

Beyond methodological contributions, \MyMethod\ offers strong practical value. Its efficiency and controllability make it well suited for {real-time} and large-scale applications such as \textbf{virtual avatar customization, digital entertainment, aesthetic enhancement systems, and human-computer interaction platforms}. 
In future work, we plan to extend the proposed method to more expressive neural network architectures, such as deep autoencoders, to further improve disentanglement and interpretability. Moreover, we aim to explore its adaptation to broader cross-domain editing tasks, including medical image analysis and AR/VR scenarios, to further demonstrate its versatility and impact in practical applications.
\section*{Declaration of competing interest}
The authors declare that they have no known competing financial interests or personal relationships that could have appeared to influence the work reported in this paper.
\section*{Acknowledgments}																																														
This work was supported by the Natural Science Foundation of Chongqing Joint Fund for Innovation and Development (Grant No. CSTB2025NSCQ-LZX0134); the Science and Research Development Center of the Ministry of Education (Grant No. 2023ZY024); the Science and Technology Research Program of Chongqing Municipal Education Commission (Grant No. KJZD-K202400809); the Natural Science Foundation of Chongqing (Grant No. CSTB2023NSCQ-LZX0029); the Graduate Innovation Project of Chongqing Technology and Business University (Grant No. yjscxx2025-269-181).

{\balance
\bibliographystyle{elsarticle-num-names}
\bibliography{UFace}

@inproceedings{RetrievFace2025,
author = {Tian, Lulu and Yao, Hongxun},
title = {RetrievFace: Retrieval-Enhanced Diffusion for Controllable Text-Guided Face Editing},
year = {2025},
booktitle = {Proceedings of the 2025 International Conference on Multimedia Retrieval},
pages = {1273-1282},
}

@inproceedings{DiT4Edit2025,
author = {Feng, Kunyu and Ma, Yue and Wang, Bingyuan and Qi, Chenyang and Chen, Haozhe and Chen, Qifeng and Wang, Zeyu},
title = {DiT4Edit: diffusion transformer for image editing},
year = {2025},
booktitle = {Proceedings of the Thirty-Ninth AAAI Conference on Artificial Intelligence and Thirty-Seventh Conference on Innovative Applications of Artificial Intelligence and Fifteenth Symposium on Educational Advances in Artificial Intelligence},
pages = {},
}

@ARTICLE{11048915,
  author={Chen, Zhichao and Yang, Jie and Chen, Lifang and Li, Fan and Feng, Zhicheng and Jia, Limin and Li, Pan},
  journal={IEEE Internet of Things Journal}, 
  title={RailVoxelDet: A Lightweight 3-D Object Detection Method for Railway Transportation Driven by Onboard LiDAR Data}, 
  year={2025},
  volume={12},
  pages={37175-37189}
  }

@article{Wei2025MagicFaceHF,
  title={MagicFace: High-Fidelity Facial Expression Editing with Action-Unit Control},
  author={Mengting Wei and Tuomas Varanka and Xingxun Jiang and Huai-Qian Khor and Guoying Zhao},
  journal={ArXiv},
  year={2025},
  volume={abs/2501.02260},
}

@inproceedings{hou2025zeroshotfaceeditingidattribute,
  author = { Hou, Yang and Wang, Minggu and Zhao, Jianjun },
  booktitle = { 2025 IEEE International Conference on Multimedia and Expo (ICME) },
  title = {Zero-shot Face Editing via ID-Attribute Decoupled Inversion},
  year = {2025},
  pages = {1-6},
}

@ARTICLE{DRL2024,
  author={Wang, Xin and Chen, Hong and Tang, Si'ao and Wu, Zihao and Zhu, Wenwu},
  journal={IEEE Transactions on Pattern Analysis and Machine Intelligence}, 
  title={Disentangled Representation Learning}, 
  year={2024},
  pages={9677-9696},
}

@article{Pearson1896Mathematical,
  title={Mathematical Contributions to the Theory of Evolution.--On a Form of Spurious Correlation Which May Arise When Indices Are Used in the Measurement of Organs},
  author={Pearson and Karl},
  journal={Proceedings of the Royal Society of London},
  volume={60},
  pages={489-498},
  year={1896},
}

@INPROCEEDINGS{cds2024,
  author={Nam, Hyelin and Kwon, Gihyun and Park, Geon Yeong and Ye, Jong Chul},
  booktitle = {Proceedings of the IEEE/CVF Conference on Computer Vision and Pattern Recognition},
  title={Contrastive Denoising Score for Text-Guided Latent Diffusion Image Editing}, 
  year={2024},
  pages={9192-9201},
}

@INPROCEEDINGS{SP-DiffFusion2025,
  author={Chang, Jia and Li, Xujun},
  booktitle={2025 5th International Symposium on Computer Technology and Information Science}, 
  title={SP-DiffFusion: Diffusion-Based Infrared-Visible Fusion via Subspace Learning}, 
  year={2025},
  pages={281-286},
}

@article{Chen2023FaceAV,
  title={Face Aging via Diffusion-based Editing},
  author={Xiangyi Chen and St'ephane Lathuiliere},
  journal={ArXiv},
  year={2023},
  volume={abs/2309.11321},
}

@article{Wang2024DiffFAEAH,
  title={DiffFAE: Advancing High-fidelity One-shot Facial Appearance Editing with Space-sensitive Customization and Semantic Preservation},
  author={Qilin Wang and Jiangning Zhang and Chengming Xu and Weijian Cao and Ying Tai and Yue Han and Yanhao Ge and Hong Gu and Chengjie Wang and Yanwei Fu},
  journal={ArXiv},
  year={2024},
  volume={abs/2403.17664},
}

@INPROCEEDINGS{DiffusionCLIP2022,
  author={Kim, Gwanghyun and Kwon, Taesung and Ye, Jong Chul},
  booktitle = {Proceedings of the IEEE/CVF Conference on Computer Vision and Pattern Recognition},
  title={DiffusionCLIP: Text-Guided Diffusion Models for Robust Image Manipulation}, 
  year={2022},
  pages={2416-2425},
}

@ARTICLE{OSR2025,
  author={Jiang, Hongxiang and Luo, Xiaoyan and Yin, Jihao and Fu, Huazhu and Wang, Fuxiang},
  journal={IEEE Transactions on Neural Networks and Learning Systems}, 
  title={Orthogonal Subspace Representation for Generative Adversarial Networks}, 
  year={2025},
  volume={36},
  pages={4413-4427},
}

@INPROCEEDINGS{dds2023,
  author={Hertz, Amir and Aberman, Kfir and Cohen-Or, Daniel},
  booktitle = {Proceedings of the IEEE/CVF International Conference on Computer Vision},
  title={Delta Denoising Score}, 
  year={2023},
  pages={2328-2337},
}

@INPROCEEDINGS{FPE2024,
  author={Liu, Bingyan and Wang, Chengyu and Cao, Tingfeng and Jia, Kui and Huang, Jun},
  booktitle = {Proceedings of the IEEE/CVF Conference on Computer Vision and Pattern Recognition},
  title={Towards Understanding Cross and Self-Attention in Stable Diffusion for Text-Guided Image Editing}, 
  year={2024},
  pages={7817-7826},
}

@article{REN2025127245,
title = {Facial attribute editing via a Balanced Simple Attention Generative Adversarial Network},
journal = {Expert Systems with Applications},
volume = {277},
pages = {127245},
year = {2025},
author = {Fanghui Ren and Wenpeng Liu and Fasheng Wang and Bo Wang and Fuming Sun},
}

@article{TIAN20181,
title = {Photo-realistic 2D expression transfer based on FFT and modified Poisson image editing},
journal = {Neurocomputing},
volume = {309},
pages = {1-10},
year = {2018},
author = {Chunna Tian and Haiyang Li and Xinbo Gao},
}

@article{DOGAN2020338,
title = {Semi-supervised image attribute editing using generative adversarial networks},
journal = {Neurocomputing},
volume = {401},
pages = {338-352},
year = {2020},
author = {Yahya Dogan and Hacer Yalim Keles},
}

@inproceedings{GANSpace2020,
title={Ganspace: Discovering interpretable gan controls},
author={H{\"a}rk{\"o}nen, Erik and Hertzmann, Aaron and Lehtinen, Jaakko and Paris, Sylvain},
booktitle={Advances in neural information processing systems},
volume={33},
pages={9841--9850},
year={2020}
}

@article{InterFaceGAN2020,
  title={InterFaceGAN: Interpreting the Disentangled Face Representation Learned by GANs},
  author={Yujun Shen and Ceyuan Yang and Xiaoou Tang and Bolei Zhou},
  journal={IEEE Transactions on Pattern Analysis and Machine Intelligence},
  year={2020},
  volume={44},
  pages={2004-2018},
}

@inproceedings{SEFA2021,
  author={Yujun Shen and Bolei Zhou},
  booktitle={Proceedings of the IEEE/CVF Conference on Computer Vision and Pattern Recognition}, 
  title={Closed-Form Factorization of Latent Semantics in GANs}, 
  year={2021},
  pages={1532-1540},
}

@inproceedings{StyleCLIP2021ICCV,
  author={Patashnik, Or and Wu, Zongze and Shechtman, Eli and Cohen-Or, Daniel and Lischinski, Dani},
  booktitle={Proceedings of the IEEE/CVF International Conference on Computer Vision}, 
  title={StyleCLIP: Text-Driven Manipulation of StyleGAN Imagery}, 
  year={2021},
  pages={2065-2074}
}

@inproceedings{ResNet2016CVPR,
  title={Deep Residual Learning for Image Recognition},
  author={He, Kaiming and Zhang, Xiangyu and Ren, Shaoqing and Sun, Jian},
  booktitle={Proceedings of the IEEE/CV Conference on Computer Vision and Pattern Recognition},
  pages={770--778},
  year={2016}
}

@inproceedings{SEResNet2018CVPR,
  title={Squeeze-and-Excitation Networks},
  author={Hu, Jie and Shen, Li and Sun, Gang},
  booktitle={Proceedings of the IEEE/CV Conference on Computer Vision and Pattern Recognition},
  pages={7132--7141},
  year={2018}
}

@article{LI201931,
title = {His-GAN: A histogram-based GAN model to improve data generation quality},
journal = {Neural Networks},
pages = {31-45},
year = {2019},
author = {Wei Li and Wei Ding and Rajani Sadasivam and Xiaohui Cui and Ping Chen},

}

@article{SHAO2021107311,
title = {DMDIT: Diverse multi-domain image-to-image translation},
journal = {Knowledge-Based Systems},
pages = {107311},
year = {2021},
author = {Mingwen Shao and Youcai Zhang and Huan Liu and Chao Wang and Le Li and Xun Shao},
}

@inproceedings{SDFlow2024ICASSP,
  author={Li, Binglei and Huang, Zhizhong and Shan, Hongming and Zhang, Junping},
  booktitle={ICASSP 2024 - 2024 IEEE International Conference on Acoustics, Speech and Signal Processing}, 
  title={Semantic Latent Decomposition with Normalizing Flows for Face Editing}, 
  year={2024},
  pages={4165-4169},
}

@inproceedings{MDSE2023,
  author={Naveh, Chen},
  booktitle={Proceedings of the IEEE/CVF International Conference on Computer Vision}, 
  title={Multi-Directional Subspace Editing in Style-Space}, 
  year={2023},
  pages={7104-7114},
  }

@inproceedings{AdaTrans2023,
    author    = {Huang, Zhizhong and Ma, Siteng and Zhang, Junping and Shan, Hongming},
    title     = {Adaptive Nonlinear Latent Transformation for Conditional Face Editing},
    booktitle = {Proceedings of the IEEE/CVF International Conference on Computer Vision},
    year      = {2023},
    pages     = {21022-21031}
}

@inproceedings{StarGAN2019,
  title={StarGAN v2: Diverse Image Synthesis for Multiple Domains},
  author={Yunjey Choi and Youngjung Uh and Jaejun Yoo and Jung-Woo Ha},
  booktitle={Proceedings of the IEEE/CVF Conference on Computer Vision and Pattern Recognition},
  year={2019},
  pages={8185-8194},
}

@inproceedings{StyleTrans2021,
author = {Lin, Peizhen and Liu, Baoyu and Wang, Lei and Lei, Zetong and Cheng, Jun},
title = {Face Translation based on Semantic Style Transfer and Rendering from One Single Image},
year = {2021},
booktitle = {Proceedings of the 2021 10th International Conference on Software and Computer Applications},
pages = {166--172},
}

@article{UVCGAN2022,
  title={UVCGAN: UNet Vision Transformer cycle-consistent GAN for unpaired image-to-image translation},
  author={Dmitrii Torbunov and Yi Huang and Haiwang Yu and Jin-zhi Huang and Shinjae Yoo and Meifeng Lin and Brett Viren and Yihui Ren},
  journal={2023 IEEE/CVF Winter Conference on Applications of Computer Vision},
  year={2022},
  pages={702-712},
}

@article{SUN2024109346,
title = {A Fine Rendering High-Resolution Makeup Transfer network via inversion-editing strategy},
journal = {Engineering Applications of Artificial Intelligence},
pages = {109346},
year = {2024},
author = {Zhaoyang Sun and Shengwu Xiong and Yaxiong Chen and Yi Rong},
}

@article{GUO2024108683,
title = {Enhancing accuracy, diversity, and random input compatibility in face attribute manipulation},
journal = {Engineering Applications of Artificial Intelligence},
pages = {108683},
year = {2024},
author = {Qi Guo and Xiaodong Gu},
}

@article{YANG2023105519,
title = {DFSGAN: Introducing editable and representative attributes for few-shot image generation},
journal = {Engineering Applications of Artificial Intelligence},
pages = {105519},
year = {2023},
author = {Mengping Yang and Saisai Niu and Zhe Wang and Dongdong Li and Wenli Du},
}

@inproceedings{ICLR2020Steer,
 author = {Jahanian, Ali and Chai, Lucy and Isola, Phillip},
 title = {On the steerability of generative adversarial networks},
 booktitle = {Proceedings of International Conference on Learning Representations},
 year = {2020}
}

@inproceedings{semanticBound2020,
title={Interpreting the latent space of gans for semantic face editing},
author={Shen, Yujun and Gu, Jinjin and Tang, Xiaoou and Zhou, Bolei},
booktitle={Proceedings of the IEEE/CVF conference on computer vision and pattern recognition},
pages={9243--9252},
year={2020}
}

@inproceedings{Karras.2020,
 title={Analyzing and improving the image quality of stylegan},
  author={Karras, Tero and Laine, Samuli and Aittala, Miika and Hellsten, Janne and Lehtinen, Jaakko and Aila, Timo},
  booktitle={Proceedings of the IEEE/CVF conference on computer vision and pattern recognition},
  pages={8110--8119},
  year={2020}
}

@article{Suo.2011,
 author = {Suo, Jinli and Lin, Liang and Shan, Shiguang and Chen, Xilin and Gao, Wen},
 year = {2011},
 title = {High-Resolution Face Fusion for Gender Conversion},
 pages = {226--237},
 volume = {41},
 journal = {IEEE Transactions on Systems, Man, and Cybernetics - Part A: Systems and Humans},
}

@inproceedings{Li.2012,
 author = {Li, Kai and Xu, Feng and Wang, Jue and Dai, Qionghai and Liu, Yebin},
 title = {A data-driven approach for facial expression synthesis in video},
 pages = {57--64},             
 booktitle = {Proceedings of the IEEE/CV Conference on Computer Vision and Pattern Recognition},
 year = {2012},
}

@inproceedings{KemelmacherShlizerman.2014,
 author = {Kemelmacher-Shlizerman, Ira and Suwajanakorn, Supasorn and Seitz, Steven M.},
 title = {Illumination-Aware Age Progression},
 pages = {3334--3341},                
 booktitle = {Proceedings of the IEEE/CV Conference on Computer Vision and Pattern Recognition},
 year = {2014},
}

@inproceedings{Garrido.2014,
 author = {Garrido, Pablo and Valgaerts, Levi and Rehmsen, Ole and Thormaehlen, Thorsten and Perez, Patrick and Theobalt, Christian},
 title = {Automatic Face Reenactment},
 pages = {4217--4224},
 booktitle = {Proceedings of the IEEE/CV Conference on Computer Vision and Pattern Recognition},
 year = {2014},
}

@inproceedings{Davis.2022,
 author = {Davis, Keith M. and de {La Torre-Ortiz}, Carlos and Ruotsalo, Tuukka},
 title = {Brain-Supervised Image Editing},
 pages = {18459--18468},
 booktitle = {Proceedings of the IEEE/CV Conference on Computer Vision and Pattern Recognition},
 year = {2022},
}

@article{AttGAN2019,
 author = {He, Zhenliang and Zuo, Wangmeng and Kan, Meina and Shan, Shiguang and Chen, Xilin},
 year = {2019},
 title = {AttGAN: Facial Attribute Editing by Only Changing What You Want},
 pages = {5464--5478},
 volume = {28},
 journal = {IEEE Transactions on Image Processing : a publication of the IEEE Signal Processing Society},
}

@article{PatchNet2013,
author = {Hu, Shi-Min and Zhang, Fang-Lue and Wang, Miao and Martin, Ralph R. and Wang, Jue},
title = {PatchNet: a patch-based image representation for interactive library-driven image editing},
year = {2013},
journal = {ACM Transactions on Graphics},
}

@INPROCEEDINGS{AGE2022CVPR,
  author={Ding, Guanqi and Han, Xinzhe and Wang, Shuhui and Wu, Shuzhe and Jin, Xin and Tu, Dandan and Huang, Qingming},
  booktitle={Proceedings of the IEEE/CV Conference on Computer Vision and Pattern Recognition}, 
  title={Attribute Group Editing for Reliable Few-shot Image Generation}, 
  year={2022},
  pages={11184-11193},
}

@inproceedings{EnjoyGAN2021,
  author = {P. Zhuang and O. Koyejo and A.~G. Schwing},
  title = {{Enjoy Your Editing: Controllable GANs for Image Editing via Latent Space Navigation}},
  booktitle = {Proceedings of International Conference on Learning Representations},
  year = {2021},
}

@INPROCEEDINGS{STGAN2019,
  title={STGAN: A Unified Selective Transfer Network for Arbitrary Image Attribute Editing},
  author={Ming Liu and Yukang Ding and Min Xia and Xiao Liu and Errui Ding and Wangmeng Zuo and Shilei Wen},
  booktitle={Proceedings of the IEEE/CV Conference on Computer Vision and Pattern Recognition},
  year={2019},
  pages={3668-3677},
}

@inproceedings{IALS2021IJCAI,
  title     = {Disentangled Face Attribute Editing via Instance-Aware Latent Space Search},
  author    = {Han, Yuxuan and Yang, Jiaolong and Fu, Ying},
  booktitle = {Proceedings of the Thirtieth International Joint Conference on
               Artificial Intelligence},
  pages     = {715--721},
  year      = {2021},
}

@article{LI2023103916,
title = {Controllable facial attribute editing via Gaussian mixture model disentanglement},
journal = {Digital Signal Processing},
volume = {134},
pages = {103916},
year = {2023},
author = {Bo Li and Shu-Hai Deng and Bin Liu and Yike Li and Zhi-Fen He and Yu-Kun Lai and Congxuan Zhang and Zhen Chen},
}

@inproceedings{Huang2024SDGANDS,
  title={SDGAN: Disentangling Semantic Manipulation for Facial Attribute Editing},
  author={Wenmin Huang and Weiqi Luo and Jiwu Huang and Xiaochun Cao},
  booktitle={AAAI Conference on Artificial Intelligence},
  year={2024},
}

@ARTICLE{TPAMI2023VecGAN,
  author={Dalva, Yusuf and Pehlivan, Hamza and Hatipoglu, Oyku Irmak and Moran, Cansu and Dundar, Aysegul},
  journal={IEEE Transactions on Pattern Analysis and Machine Intelligence}, 
  title={Image-to-Image Translation With Disentangled Latent Vectors for Face Editing}, 
  year={2023},
  volume={45},
  pages={14777-14788},
}

@article{TIP2022STIAWO,
 author = {Liu, Kanglin and Cao, Gaofeng and Zhou, Fei and Liu, Bozhi and Duan, Jiang and Qiu, Guoping},
 year = {2022},
 title = {Towards Disentangling Latent Space for Unsupervised Semantic Face Editing},
 pages = {1475--1489},
 volume = {31},
 journal = {IEEE Transactions on Image Processing },
}

@article{MaskFaceGAN2023,
  author={Pernuš, Martin and Štruc, Vitomir and Dobrišek, Simon},
  journal={IEEE Transactions on Image Processing}, 
  title={MaskFaceGAN: High-Resolution Face Editing With Masked GAN Latent Code Optimization}, 
  year={2023},
  volume={32},
  pages={5893-5908},
}

@INPROCEEDINGS{Yao2021ICCV,
  author={Yao, Xu and Newson, Alasdair and Gousseau, Yann and Hellier, Pierre},
  booktitle={Proceedings of the IEEE/CVF International Conference on Computer Vision}, 
  title={A Latent Transformer for Disentangled Face Editing in Images and Videos}, 
  year={2021},
  pages={13769-13778},
}

@inproceedings {Van2021,
author = {T. Van and T. Nguyen and N. N. Tran and H. Nguyen and L. Doan and H. Dao and T. Minh},
booktitle = {2020 International Conference on Advanced Computing and Applications},
title = {Interpreting the Latent Space of Generative Adversarial Networks using Supervised Learning},
year = {2020},
pages = {49-54},
}

@inproceedings{VoynovB20,
  author    = {Andrey Voynov and                Artem Babenko},
  title     = {Unsupervised Discovery of Interpretable Directions in the GAN Latent Space},
  booktitle = {Proceedings of International Conference on Machine Learning},
  volume    = {119},
  pages     = {9786--9796},
  year      = {2020}
 }

@inproceedings{Cherepkov.2021,
 author = {Cherepkov, Anton and Voynov, Andrey and Babenko, Artem},
 title = {Navigating the GAN Parameter Space for Semantic Image Editing},
 pages = {3670--3679},
 booktitle = {Proceedings of the IEEE/CV Conference on Computer Vision and Pattern Recognition},
 year = {2021},
}

@inproceedings{ICLR2021Regres,
title={Using latent space regression to analyze and leverage compositionality in GAN},
author={Lucy Chai and Jonas Wulff and Phillip Isola},
booktitle={Proceedings of International Conference on Learning Representations},
year={2021},
}

@inproceedings{ICLR2022Escape,
title={Do Not Escape From the Manifold: Discovering the Local Coordinates on the Latent Space of {GAN}s},
author={Jaewoong Choi and Junho Lee and Changyeon Yoon and Jung Ho Park and Geonho Hwang and Myungjoo Kang},
booktitle={Proceedings of International Conference on Learning Representations},
year={2022},
}

@inproceedings{SOCFS_CVPR_2015,
 author = {Han, Dongyoon and Kim, Junmo},
 title = {Unsupervised Simultaneous Orthogonal basis Clustering Feature Selection},
 pages = {5016--5023},
 booktitle = {Proceedings of the IEEE/CV Conference on Computer Vision and Pattern Recognition},
 year = {2015},
}

@inproceedings{StyleGAN2019,
  title={A style-based generator architecture for generative adversarial networks},
  author={Tero Karras and Samuli Laine and Timo Aila},
  booktitle={Proceedings of the IEEE/CV Conference on Computer Vision and Pattern Recognition},
  pages={4401--4410},
  year={2019}
}

@inproceedings{ProGAN2018,
      title={Progressive Growing of GANs for Improved Quality, Stability, and Variation}, 
      author={Tero Karras and Timo Aila and Samuli Laine and Jaakko Lehtinen},
      year={2018},
      booktitle={Proceedings of International Conference on Learning Representations},
}

@article{TPAMI2024FGE,
  author       = {Andrew Melnik and
                  Maksim Miasayedzenkau and
                  Dzianis Makarovets and
                  Dzianis Pirshtuk and
                  Eren Akbulut and
                  Dennis Holzmann and
                  Tarek Renusch and
                  Gustav Reichert and
                  Helge J. Ritter},
  title        = {Face Generation and Editing With StyleGAN: {A} Survey},
  journal      = {IEEE Transactions on Pattern Analysis and Machine Intelligence},
  volume       = {46},
  pages        = {3557--3576},
  year         = {2024},
}

@inproceedings{NIPS2017FID,
 author = {Heusel, Martin and Ramsauer, Hubert and Unterthiner, Thomas and Nessler, Bernhard and Hochreiter, Sepp},
 booktitle = {Proceedings of Advances in Neural Information Processing Systems},
 pages = {pages = {6629-6640}},
 title = {GANs Trained by a Two Time-Scale Update Rule Converge to a Local Nash Equilibrium},
 volume = {30},
 year = {2017}
}

@INPROCEEDINGS{CVPR2018LPIPS,
  author={Zhang, Richard and Isola, Phillip and Efros, Alexei A. and Shechtman, Eli and Wang, Oliver},
  booktitle={Proceedings of the IEEE/CVF Conference on Computer Vision and Pattern Recognition}, 
  title={The Unreasonable Effectiveness of Deep Features as a Perceptual Metric}, 
  year={2018},
  pages={586-595}
  }

@ARTICLE{TIP2004SSIM,
  author={Zhou Wang and Bovik, A.C. and Sheikh, H.R. and Simoncelli, E.P.},
  journal={IEEE Transactions on Image Processing}, 
  title={Image quality assessment: from error visibility to structural similarity}, 
  year={2004},
  volume={13},
  pages={600-612}
 }

@inproceedings{NIP2016IS,
 author = {Salimans, Tim and Goodfellow, Ian and Zaremba, Wojciech and Cheung, Vicki and Radford, Alec and Chen, Xi},
 title = {Improved techniques for training GANs},
 year = {2016},
 pages = {2234-2242},
 booktitle = {Proceedings of Advances in Neural Information Processing Systems}
 }

@article{KBS2024ANSC,
author = {Bo Liu and Wenbo Li and Jie Li and Xuan Cui and Chongwen Liu and Hongping Gan},
title = {Attention non-negative spectral clustering},
journal = {Knowledge-Based Systems},
volume = {294},
pages = {111695},
year = {2024}
}

@inproceedings{AgeEstimationCVPR2019,
  author = {Zhang, Chao and Liu, Shuaicheng and Xu, Xun and Zhu, Ce},
  booktitle = {Proceedings of the IEEE/CVF Conference on Computer Vision and Pattern Recognition},
  pages = {12587--12596},
  title = {C3AE: Exploring the Limits of Compact Model for Age Estimation},
  year = {2019}
}
}
\clearpage

\appendix
\setcounter{figure}{0} 
\setcounter{thm}{0} 
\section{Supplementary Material for the Proof of Theorem~\ref{thm:Convergence} in main text}
\label{appendixB}
\begin{thm}[Monotone descent and first-order stationarity]
  {Let $\{(\B{W}^t,\B{P}^t,\B{F}^t)\}$ be the sequence generated by \MyAlgorithm\ with $\alpha>0$ and $\lambda\ge 0$. Then the objective values are non-increasing:}
    { 
    \begin{equation}
    J(\B{W}^{t+1},\B{P}^{t+1},\B{F}^{t+1}) \le J(\B{W}^t,\B{P}^{t},\B{F}^{t})\quad\text{for all } t.
    \end{equation}}
  \end{thm}
  
  \begin{proof}
  In the $t+1$ iteration, with $\B{F}^t$ and $\B{P}^t$ fixed, $\B{W}$ is updated using Eq.~\eqref{derive_W}, satisfying the following inequality: 
  {
    \begin{equation}
      \label{convergence_W}
      \begin{split}
      &\B{W}^{t+1}=\arg\min_{\B{W}^\top\B{W}=\B{I}} J(\B{W}^t,\B{P}^t,\B{F}^t)\\
      &\quad\quad\quad\Rightarrow
      J(\B{W}^{t+1},\B{P}^{t},\B{F}^{t})\le J(\B{W}^t,\B{P}^{t},\B{F}^{t}). \\
      \end{split}
    \end{equation}}
  With $\B{W}^{t+1}$ and $\B{F}^{t}$ fixed, $\B{P}^t$ is updated using Eq.~\eqref{calculate_P0}, satisfying the following inequality:
  {
    \begin{equation}
      \label{convergence_P}
      \begin{split}
      &\B{P}^{t+1}=\arg\min_{\B{P}} J(\B{W}^{t+1},\B{P},\B{F}^{t}) =\arg\min_{\B{P}} \|\B{Z}-{{\B{W}}^{t+1}\B{P}}\|_\text{F}^2\\
      &\quad\quad\quad\Rightarrow
      J(\B{W}^{t+1},\B{P}^{t+1},\B{F}^{t})\le J(\B{W}^{t+1},\B{P}^{t},\B{F}^{t}). \\
      \end{split}
    \end{equation}}
  With $\B{W}^{t+1}$ and $\B{P}^{t+1}$ fixed, $\B{F}$ is updated using Eq.~\eqref{objFunForF}, satisfying the following inequality.
    {
    \begin{equation}
    \label{convergence_F}
      \begin{split}
      &\B{F}^{t+1}=\arg\min_{\B{F}\ge 0} J(\B{W}^{t+1},\B{P}^{t+1},\B{F})= \arg\min_{\B{F}\ge 0} \| \B{F} - \B{W}^{t+1} \|_\text{F}^2\\
      &\quad\quad\Rightarrow 
      J(\B{W}^{t+1},\B{P}^{t+1},\B{F}^{t+1})\le J(\B{W}^{t+1},\B{P}^{t+1},\B{F}^{t}).\\ 
      \end{split}
    \end{equation}}
  By combining Eq.~\eqref{convergence_W} and Eq.~\eqref{convergence_P}, the following inequality is derived:
      { 
      \begin{equation}
      \label{convergence_W_P}
        J(\B{W}^{t+1},\B{P}^{t+1},\B{F}^{t}) \le J(\B{W}^{t+1},\B{P}^{t},\B{F}^{t}) \le J(\B{W}^t,\B{P}^{t},\B{F}^{t}).
      \end{equation}}
  Similarly, the following inequality applies for Eq.~\eqref{convergence_F} and Eq.~\eqref{convergence_W_P}:
    {
    \begin{equation}
      \label{proof_result}
      J(\B{W}^{t+1},\B{P}^{t+1},\B{F}^{t+1}) \le J(\B{W}^{t+1},\B{P}^{t+1},\B{F}^{t}) \le J(\B{W}^t,\B{P}^{t},\B{F}^{t}).
    \end{equation}}
  As shown in Eq.~\eqref{proof_result}, updating $\B{W}$, $\B{P}$ and $\B{F}$ according to \MyAlgorithm\ Algorithm guarantees that the objective in Eq.~\eqref{JWEq} monotonically decreases. 
  \end{proof}
\section{Supplementary Material for Theorem~\ref{thm:equivalent}}
\label{appendixA}
\begin{lemma}\label{lemma1}
  For a given matrix $\B{A} \in \C{R}^{m \times k}$, we consider the following optimization problem, where $\B{W}\in \C{R}^{m \times k}$ is the optimization variable, subject to orthogonal constraints:
{
  \begin{equation} 
   \label{lemmaEq1}
   \begin{split}
    \min_\B{W} \operatorname{\text{Tr}}(-\B{A}^\top \B{W}), \quad s.t \quad  \B{W}^\top\B{W}=\B{I},
   \end{split}
 \end{equation}}
 where Tr~(.) represents the trace of a matrix. Based on this operation, Eq.~\eqref{lemmaEq1} can be equivalently expressed as:
{
 \begin{equation} 
   \label{lemmaEq2}
   \begin{split}
    \min_\B{W}\|\B{A}-\B{W}\|_\text{F}^2, \quad s.t \quad  \B{W}^\top\B{W}=\B{I}.
 \end{split}
 \end{equation}}
\end{lemma}

\begin{proof}
 As $\B{A}^\top \B{A}$ and $\B{W}^\top\B{W}$ are constants, these terms can be incorporated into Eq.~\eqref{lemmaEq1}, yielding the following equation, which shares the same optimal solution as Eq.~\eqref{lemmaEq1}:
\begin{equation}
\label{lemmaEq3}
\begin{split}
  & \min_\B{W} \operatorname{\text{Tr}}(-\B{A}^\top \B{W}), \\
  & \Longleftrightarrow \min_\B{W}   \operatorname{\text{Tr}}(\B{A}^\top \B{A}-2\B{A}^\top \B{W}+\B{W}^\top \B{W}),\\
  & \Longleftrightarrow \min_\B{W}  \operatorname{\text{Tr}}( {(\B{A}-\B{W})}^\top(\B{A}-\B{W})), \\
  & \Longleftrightarrow \min_\B{W}  \|\B{A}-\B{W}\|_\text{F}^2, \\ 
  & s.t \quad \B{W}^\top\B{W}=\B{I}.
\end{split}
\end{equation}
\end{proof}

\begin{thm}
  \label{thm:equivalent}
When $\B{P}$ and $\B{F}$ are fixed, Eq.~\eqref{updateW} is equivalent to Eq.~\eqref{derive_W}.
\end{thm}
\begin{proof}
The Frobenius norm of matrix $\B{A}$ can be expressed in terms of the matrix trace as $\|\B{A}\|_\text{F}^2 = \operatorname{\text{Tr}}(\B{A}^\top \B{A})$. Consequently, Eq.~\eqref{updateW} can be reformulated as:
\hspace*{-5pt}
{
   \begin{equation}
    \label{equiv:1}
    \begin{split}
    &\min_\B{W} \text{Tr} ({\left( \B{Z}-\B{W}\B{P} \right)}^{\top}\left(\B{Z}-\B{W}\B{P} \right)) + \alpha \text{Tr} \left( \B{W}-\B{F}\right)^{\top}\left(\B{W}-\B{F} \right) \\
    &\min_\B{W} \text{Tr} ({\left( \B{Z}-\B{W}\B{P} \right)}^{\top}\left(\B{Z}-\B{W}\B{P} \right)) + \alpha \text{Tr} (\left( \B{W}-\B{F} \right)^{\top}\left(\B{W}-\B{F} \right)) \\
    &\quad\quad - \lambda \text{Tr}(\left(\B{W}-\B{S}\right)^{\top}\left(\B{W}-\B{S} \right) ), \\
    & s.t \quad \B{W}^\top\B{W}=\B{I}.
  \end{split}
\end{equation}}

Eq.~\eqref{equiv:1} is simplified as follows:
  {
  \begin{equation}
    \label{equiv:2}
    \begin{split}
    &\min_\B{W} \text{Tr}\left( {\B{Z}^\top}\B{Z}\right)-2\text{Tr}\left(\B{Z}^\top \B{W}\B{P}\right)+\text{Tr}\left((\B{W}\B{P})^\top\B{W}\B{P} \right)\\
    &\quad\quad +\alpha  \text{Tr}\left( {\B{F}^\top}\B{F}\right)-2\alpha \text{Tr}\left({\B{W}^\top}\B{F}\right)+ \alpha \text{Tr}\left({\B{W}^\top}\B{W} \right)\\
    &\quad\quad - \lambda \text{Tr}\left( {\B{S}^\top}\B{S}\right)+2 \lambda \text{Tr}\left({\B{W}^\top}\B{S}\right)- \lambda \text{Tr}\left({\B{W}^\top}\B{W} \right),\\
    & s.t \quad \B{W}^\top\B{W}=\B{I}.
   \end{split}
   \end{equation}}
 For Eq.~\eqref{equiv:2}, we derive the following two facts:
 \begin{itemize}
   \item Since $\B{W}^\top\B{W}=\B{I}$, we have: 
   \begin{equation*}
   \text{Tr}\left((\B{W}\B{P})^\top\B{W}\B{P} \right)=\text{Tr}\left(\B{P}^\top\B{W}^\top\B{W}\B{P} \right) = \operatorname{\text{Tr}}(\B{P}^\top \B{P}).
   \end{equation*}
    When $\B{P}$ and $\B{F}$ are fixed, $\operatorname{\text{Tr}}(\B{P}^\top \B{P})$ and $\operatorname{\text{Tr}}(\B{F}^\top \B{F})$ are constant.
   \item Since $\B{Z}$ and $\B{S}$ are given matrices, $\operatorname{\text{Tr}}(\B{Z}^\top \B{Z})$ and $\operatorname{\text{Tr}}(\B{S}^\top\B{S})$ are constants. 
 \end{itemize}
 Based on the above two facts, by eliminating the terms $\operatorname{\text{Tr}}(\B{Z}^\top \B{Z})$, $\operatorname{\text{Tr}}(\B{S}^\top\B{S})$, $\operatorname{\text{Tr}}(\B{P}^\top \B{P})$ and $\operatorname{\text{Tr}}(\B{F}^\top \B{F})$, Eq.~\eqref{equiv:2} can be further simplified as follows: 
  \begin{equation}
    \label{equiv:3}
    \begin{split}
    & \min_\B{W} \text{Tr} \left(-\left( \B{Z}^\top \B{P}+\alpha \B{F}^\top-\lambda{\B{W}\B{S}^\top} \right)^\top \B{W} \right),\\
    & s.t \quad \B{W}^\top\B{W}=\B{I}.
    \end{split}
  \end{equation}
  As shown in Lemma~\ref{lemma1}, Eq.~\eqref{equiv:3} is equivalent to Eq.~\eqref{derive_W}. Therefore, Eq.~\eqref{updateW} and Eq.~\eqref{derive_W} are also equivalent.
\end{proof}

\section{Additional Experiments on Identity Consistency for \MyMethod.}
\label{sec:identity_consistency}

\subsection{Experimental Design}
To analyze the performance of the semantic vectors computed by \MyMethod\ in preserving identity consistency during facial image editing, we use the identity consistency score (referred to as $IDS$), computed according to the following formula~\citep{StyleCLIP2021ICCV}:
\begin{equation}
  \label{IDScore}
  IDS^j(k)=1 - \mathcal{F}(M_{ori}^j(k)) \cdot \mathcal{F}(M_{edit}^j(k)),
  \end{equation}
where $M_{ori}^j(k)$ and $M_{edit}^j(k)$ are the $j$-th original facial images and the corresponding edited images generated by the pre-trained generator StyleGAN2. The notation $\mathcal{F}(\cdot)$ represents the pre-trained face recognition model SE-ResNet-50~\citep{SEResNet2018CVPR}. Lower $IDS$ indicates stronger preservation of identity consistency. The average $IDS$ is used as the final measure of identity consistency.\par
We compare four methods (\sefa, \GANspace, \enjoyGAN, and \InterFaceGAN) with \MyMethod\ to analyze identity changes during facial editing.
\subsection{Comparative Analysis} 
Fig.~\ref{fig:id_consistency} illustrates that identity consistency decreases as editing magnitude increases, with larger magnitudes leading to more significant changes. At the same level of editing magnitude, \MyMethod\ outperforms others in terms of Age attribute.\par
For Gender attribute, \MyMethod\ achieves higher identity similarity than \sefa, \GANspace, and \InterFaceGAN\ when the editing magnitude is less than $0.2$, and surpasses all other algorithms when it exceeds $0.4$. A similar result is observed for Hairstyle attribute. For Smile attribute, \MyMethod\ consistently outperforms \sefa, \InterFaceGAN, and \enjoyGAN\ when the magnitude exceeds $0.2$.\par
\begin{figure*}[h]
  \centering

  \includegraphics[width=1\textwidth]{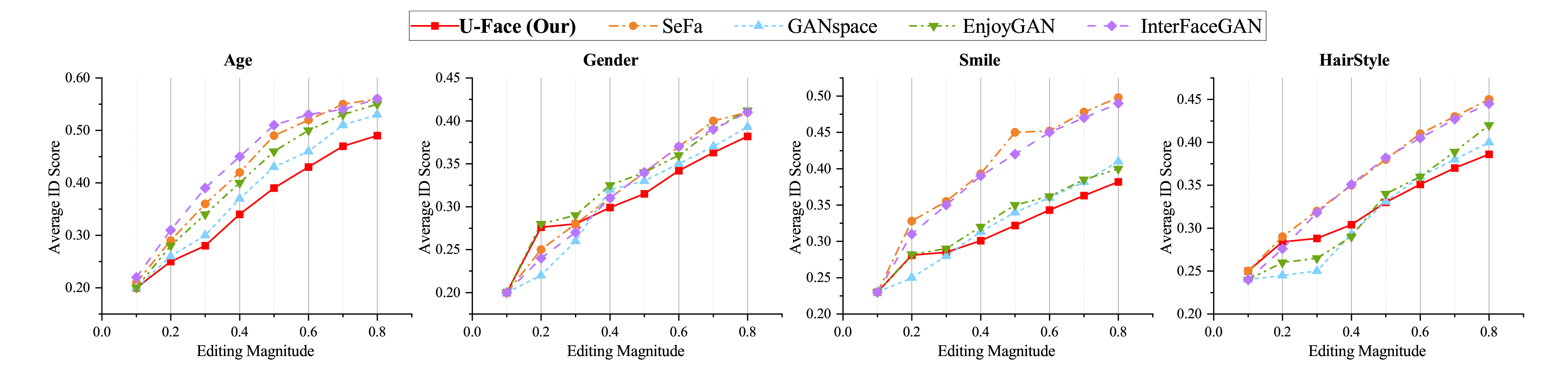}
  \setlength{\abovecaptionskip}{-10pt}
  \caption{Comparison of identity consistency preservation across \sefa, \GANspace, \enjoyGAN, \InterFaceGAN, and \MyMethod\ using the $IDS$ metric for the Age, Gender, Smile, and Hairstyle attributes. The editing magnitude $\beta$ varies from $0.1$ to $0.8$.}
~\label{fig:id_consistency}
\end{figure*}
In summary, the proposed \MyMethod\ framework maintains high identity consistency across Age, Gender, Hairstyle, and Smile attributes. Compared to the baseline methods, \MyMethod\ performs better as the editing magnitude increases, ensuring that the edited facial attributes remain well aligned with the original facial image. This is particularly crucial in applications such as virtual try-ons and character synthesis, where it guarantees that the edited face appears natural, realistic, and easily recognizable.

\section{The Calculation Method of The Pearson Correlation Coefficient.}
\label{sec:corrcompute}
We compute attribute scores for the $k$-th facial attribute on both the original facial image $M_{ori}^j(k)$ and the edited facial image $M_{edit}^j(k)$ using the $k$-th attribute classifier $C_k$~\citep{ResNet2016CVPR}. These scores are defined as $S_{ori}^j(k)=C_k\left(M_{ori}^j(k)\right)$ and $S_{edit}^j(k)=C_k(M_{edit}^j(k))$, respectively, with a difference $\Delta^j(k) = S_{ori}^j(k) - S_{edit}^j(k)$. These differences are then used to compute the Pearson correlation coefficient~\citep{Pearson1896Mathematical} between the $i$-th and $k$-th attributes, as follows:
\begin{equation}
  \label{corrcompute}
  corr(i,k) = \left| \frac{cov(\Delta(i), \Delta(k))}{\sigma(\Delta(i)) \sigma(\Delta(k))} \right|,
\end{equation}
where $\Delta(i)=\{ \Delta^1(i),\Delta^2(i),\ldots,\Delta^n(i) \}$ and $\Delta(k)=\{\Delta^1(k),\Delta^2(k),\ldots,\Delta^n(k)\}$ are the sets of attribute differences for $i$-th and $k$-th attributes, respectively, $n$ is the number of edited images. $\sigma$ is the standard deviation, $cov(.)$ represents the covariance, and $|\cdot|$ indicates the absolute value.\par

\section{Ethical Considerations}
{\textbf{Scope and intent.} U-Face is intended for research, virtual avatars, digital entertainment, and human-computer interaction. It is not designed for surveillance, biometric identification, or profiling.}\par
{\textbf{Data and privacy.} Our experiments rely on publicly available face datasets (e.g., FFHQ, CelebA-HQ) and GAN-generated images under their licenses. We do not collect private images. No biometric identification is performed, and all results are reported at an aggregate level.}\par
{\textbf{User study ethics.} For the user study (20 participants), all participants provided informed consent; responses were anonymous and no personal information or participant images were stored.}\par
{\textbf{Bias and fairness.} Face datasets may exhibit demographic imbalance; attribute editing (e.g., ``gender'') is treated as visual presentation rather than identity. We report limitations arising from dataset coverage and discuss failure modes under large edits or attribute composition in Section~\ref{sec:limitations}.}\par
{\textbf{Misuse risks and safeguards.} Facial attribute editing can be misused for identity manipulation or disinformation. We recommend, for deployment contexts, (i) consent requirements when editing real-person images, (ii) explicit disclosure that images are edited, (iii) content provenance or watermarking where feasible, and (iv) reasonable bounds on edit magnitude.}\par
{\textbf{Scope of intended use}. We intend to apply this method to research, avatars, and HCI applications, and not for surveillance or biometric profiling.}\par
{\textbf{Transparency.} We document the evaluation protocol and limitations; any future releases should include clear usage guidelines and respect dataset licenses and local regulations.}

\end{document}